\def\year{2020}\relax
\documentclass[letterpaper]{article} 
\usepackage{aaai20}  
\usepackage{times}  
\usepackage{helvet} 
\usepackage{courier}  
\usepackage[hyphens]{url}  
\usepackage{graphicx} 
\usepackage{graphicx}  
\usepackage{nicefrac} 
\usepackage{multirow}
\usepackage{amsmath}
\usepackage{amsthm}
\usepackage{amssymb}
\usepackage{bm} 
\usepackage{subcaption}
\usepackage{array}
\DeclareMathOperator*{\argmin}{arg\,min}

\newtheorem{theorem}{Theorem}

\newtheorem*{remark}{Remark}
\newtheorem{lemma}{Lemma}
\newtheorem{definition}{Definition}
\newcommand{\PreserveBackslash}[1]{\let\temp=\\#1\let\\=\temp}
\newcolumntype{C}[1]{>{\PreserveBackslash\centering}p{#1}}
\newcolumntype{R}[1]{>{\PreserveBackslash\raggedleft}p{#1}}
\newcolumntype{L}[1]{>{\PreserveBackslash\raggedright}p{#1}}
\nocopyright
\title{Solving General Elliptical Mixture Models through an Approximate Wasserstein Manifold}
\author{Shengxi Li\textsuperscript{\rm 1}\thanks{We thank the Imperial Lee Family Scholarship Funding for the support.}, Zeyang Yu\textsuperscript{\rm 1}, Min Xiang\textsuperscript{\rm 1}, Danilo Mandic\textsuperscript{\rm 1}\\ 
\textsuperscript{\rm 1}Imperial College London\\ 
South Kensington Campus\\
London SW7 2AZ, UK\\
\{shengxi.li17, z.yu17, m.xiang13, d.mandic\}@imperial.ac.uk 
}
 
\begin{document}

\maketitle

\begin{abstract}
We address the estimation problem for general finite mixture models, with a particular focus on the elliptical mixture models  (EMMs).  Compared to the widely adopted Kullback–Leibler divergence, we show that the Wasserstein distance provides a more desirable optimisation space. We thus provide a stable solution to the EMMs that is both robust to initialisations and reaches a superior optimum by adaptively optimising along a manifold of an approximate Wasserstein distance. To this end, we first provide a unifying account of computable and identifiable EMMs, which serves as a basis to rigorously address the underpinning optimisation problem. Due to a probability constraint, solving this problem is extremely cumbersome and unstable, especially under the Wasserstein distance. To relieve this issue, we introduce an efficient optimisation method on a statistical manifold defined under an approximate Wasserstein distance, which allows for explicit metrics and computable operations, thus significantly stabilising and improving the EMM estimation. We further propose an adaptive method to accelerate the convergence. Experimental results demonstrate the excellent performance of the proposed EMM solver. 
\end{abstract}

\section{Introduction}
\begin{figure*}[!htb]
	\begin{center}
		\subcaptionbox{Simple GMM}{\includegraphics[width=0.14\textwidth]{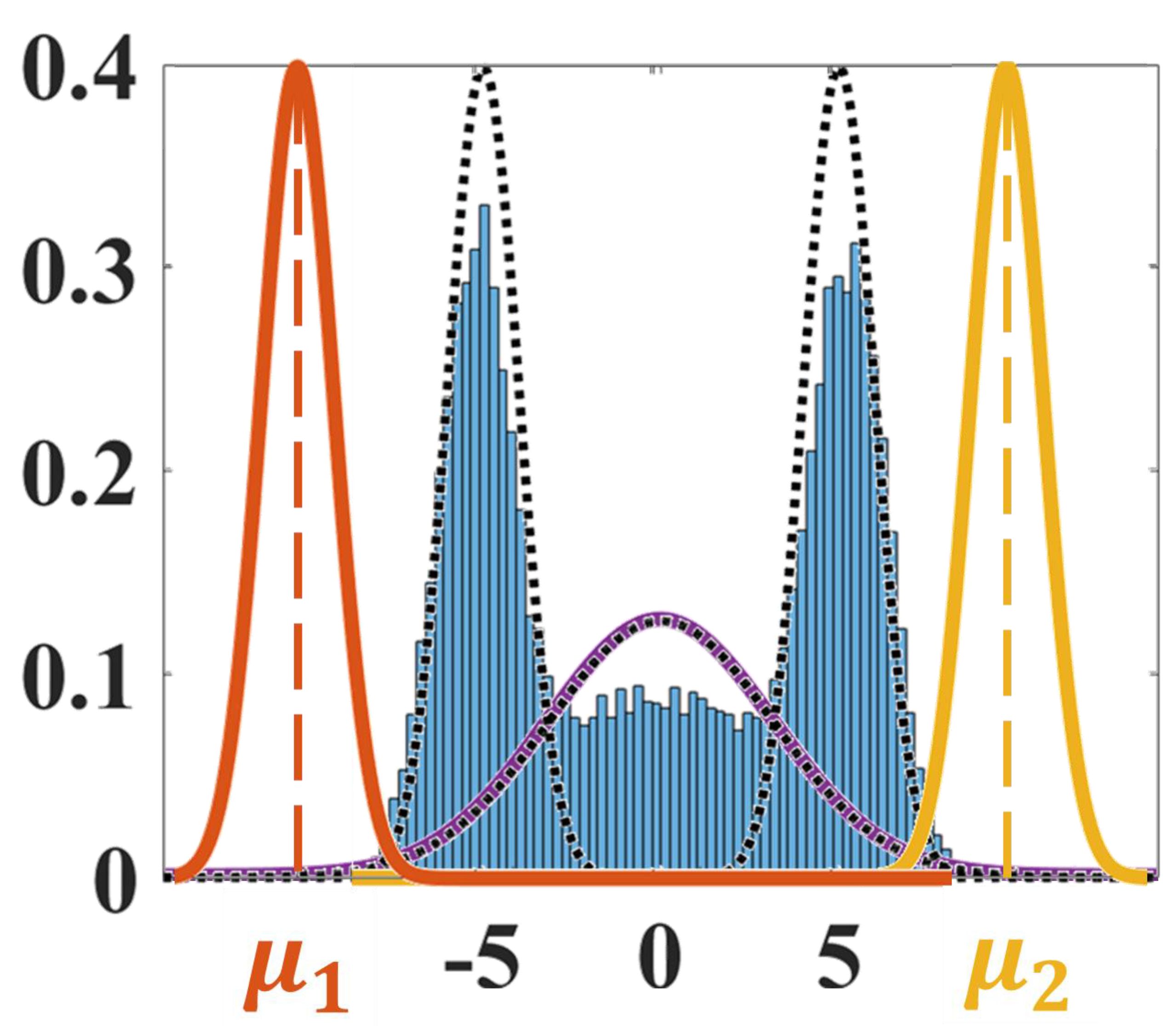}}\hspace{1.5em}
		\subcaptionbox{KL cost space}{\includegraphics[width=0.33\textwidth]{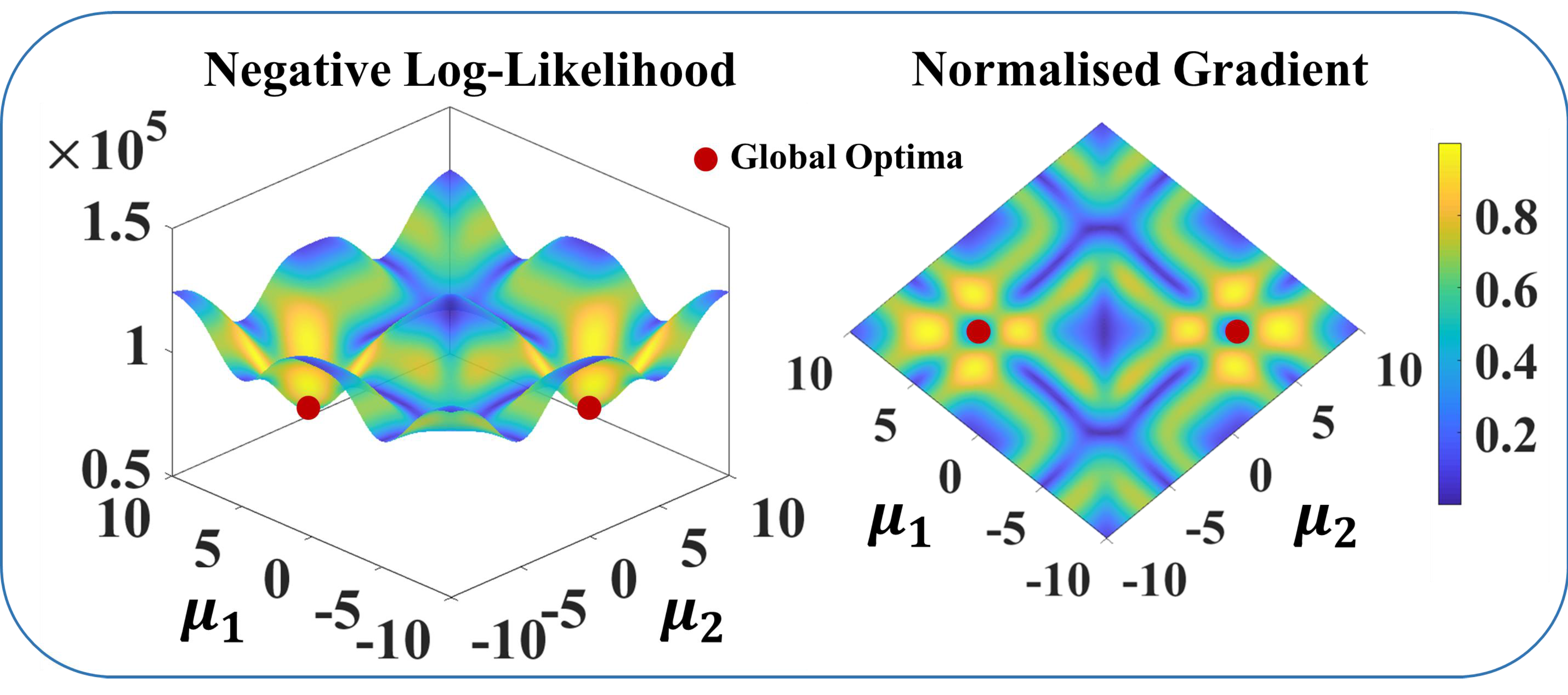}}\hspace{1.5em}
		\subcaptionbox{Wasserstein cost space}{\includegraphics[width=0.33\textwidth]{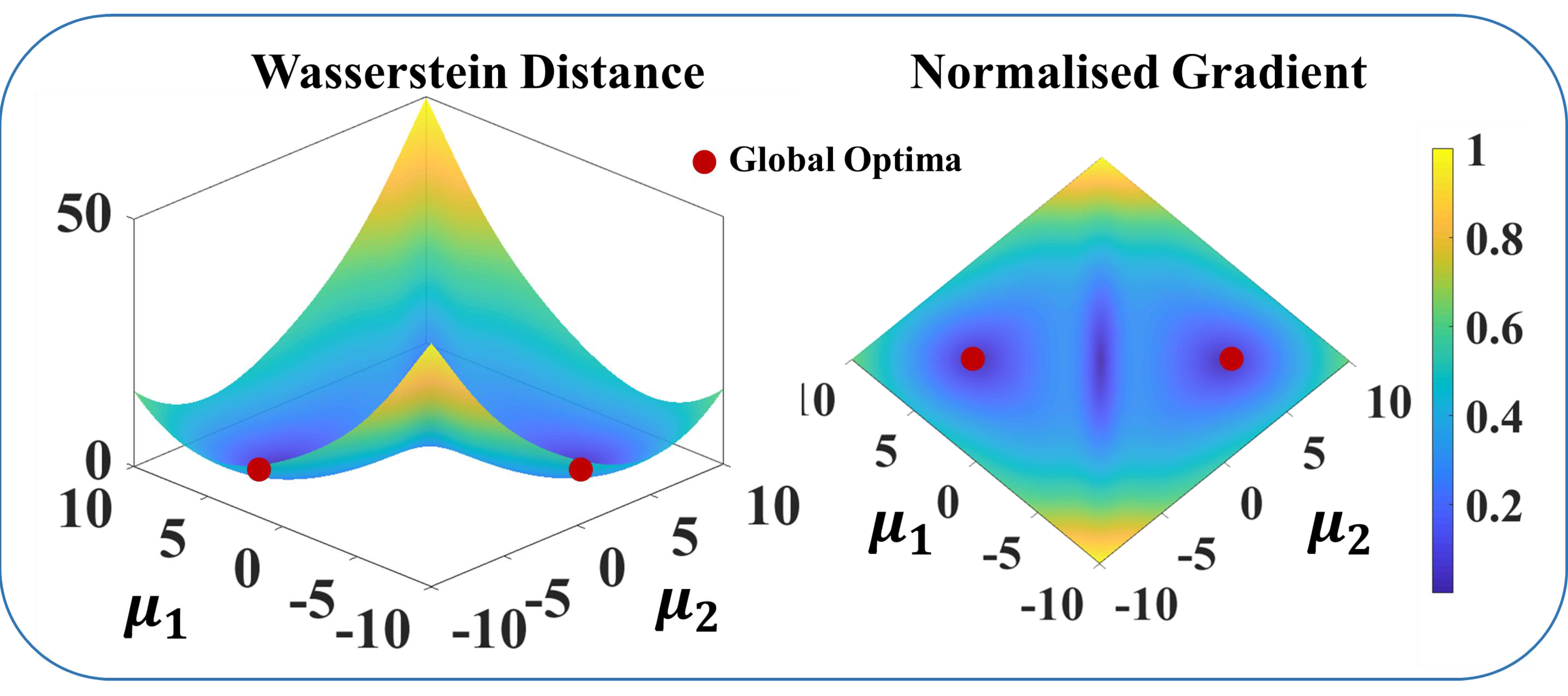}}
	\end{center}
	\caption{Illustration of a simple one-dimensional GMM with 3 clusters. (a): The ground-truth of clusters is plotted via the black dotted lines; Parametric models $\rho(\bm{\theta})$ are denoted as colourful lines; We set ALL the parameters $\bm{\theta}$ of the GMM to the ground-truth EXCEPT the mean values of only two clusters, i.e., $\mu_1$ and $\mu_2$. (b): The KL divergence and the normalised gradient versus the variance of $\mu_1$ and $\mu_2$. (c): The Wasserstein distance and the normalised gradient versus the variance of $\mu_1$ and $\mu_2$.}\label{Motiva_Illus}
\end{figure*}
This paper establishes a general solution to the finite mixture model problem, which has been attracting extensive research effort for decades, due to both its simple representation and potential for universal approximation on arbitrary distributions in $\mathbb{R}^M$. The finite mixture model also provides interpretable and statistical descriptions of data, which makes it a popular choice in a wide range of statistical learning paradigms, such as semi-supervised learning, capsule networks, and various image processing paradigms (e.g., de-noising, matching and registration). 

The estimation on a finite mixture model boils down to a minimisation problem which considers the mixture of distributions as a parametric model $\rho(\bm{\theta})$, which is then optimised through a minimisation of a certain discrepancy measure between $\rho(\bm{\theta})$ and empirical distributions of observed data $\rho^*$, namely, $\min_{\bm\theta} d\left(\rho(\bm{\theta}), \rho^*\right)$. This minimisation, although not explicitly stated, is a constrained problem because $\rho(\bm{\theta})$ must maintain the property of a probability density throughout, to ensure that $d(\cdot,\cdot)$ is tractable. 

Due to this probability constraint, various advanced numerical algorithms (solvers) have been typically restricted by either the requirement of an increasingly flexible $\rho(\bm{\theta})$ or a powerful $d(\cdot,\cdot)$. Such  restrictions, for example, are one of the main rationales for using the expectation-maximisation (EM) algorithm to minimise the Kullback–Leibler (KL) divergence in Gaussian mixture model (GMM) problems \cite{xu1996convergence}. On the other hand,  gradient-based numerical algorithms typically rest upon additional techniques that only work in particular situations (e.g., gradient reduction \cite{redner1984mixture}, positive definite projection \cite{hosseini2015matrix}, re-parametrisation \cite{jordan1994hierarchical} and Cholesky decomposition \cite{naim2012convergence}). Besides the GMM, there exist other solutions that allow for flexible choices of $\rho(\bm{\theta})$, which still belong to the EM-type methods (e.g., mixtures of t-distributions \cite{peel2000robust}, Laplace distributions \cite{tan2007multivariate} and hyperbolic distributions \cite{browne2015mixture}). Unfortunately, given other suitable candidates of distributions, those EM-type methods cannot ensure universal convergence \cite{kent1991redescending,zhang2013multivariate,sra2013geometric}, which dramatically limits the power of finite mixture models.

Another issue that has been highlighted in the literature is the sensitivity to initialisations when solving GMMs \cite{xu2016global,jin2016local}. One of the main reasons is due to the use of KL divergence, which operates based on a ``bin-to-bin'' comparison between two density histograms. This means that mixtures which fall into a spurious local minimum cannot be corrected via the points outside. Indeed, with random initialisations for the GMM, Jin has proved that the EM algorithm or any other first-order method which minimises the KL divergence are highly likely to result in arbitrary bad local minima \cite{jin2016local}. This can also be easily verified from the non-smooth optimization space with various local optimum of the KL divergence as illustrated in Fig. \ref{Motiva_Illus}-(b), even for estimating a simple GMM (Fig. \ref{Motiva_Illus}-(a)). The gradients on the space are highly concentrated as well, which also leads to ill-posed gradient descent.

On the other hand, by virtue of the reflection of sample space \cite{kolouri2017optimal} within the Wasserstein distance\footnote{Throughout this paper, the term Wasserstein distance refers to the square-Wasserstein distance.} \cite{monge1781memoire} which employs a ``cross-bin'' comparison, many practical benefits may be achieved in learning tasks \cite{arjovsky2017wasserstein}. This property is particularly appealing in mixture model problems, where it ideally provides a comprehensive distance measure over all possible transport plans. The optimization space of the Wasserstein distance is also shown in Fig. \ref{Motiva_Illus}-(c), of which smoothness is witnessed. Basically, there only exist global optimum in this case. The gradients on this space are also well behavioured, which promises to achieve superior optimisation. Most recently, Kolouri \textit{et al.} \cite{kolouri2018sliced} have adopted the Wasserstein distance for solving GMM problems. However, the aforementioned probability constraint enforces an extremely small stepsize (learning rate) during optimisation. Given the random projections of the sliced Wasserstein distance, this setup leads to extremely slow convergence or even non-convergence (especially in high dimensions as shown in our experiments). At the end of their optimisation, the EM algorithm is still needed to stabilised the algorithm.

\textbf{Motivations and Contributions:}
Despite extensive benefits, optimising the GMM under the Wasserstein distance is still not feasible due to the probability constraint. To address this from a different perspective, 
we resort to the statistical manifold. We should further point out that our work is different from information geometry \cite{amari1998natural}, as directly establishing a manifold in the whole density space of mixture models is absolutely intractable and cumbersome in optimisation \cite{wang2015discriminant}. Another problem within the Wasserstein space is that the geodesic between two mixture models may not lie in mixture models of the same type \cite{chen2018natural}, which leads to non-convergence.

We propose to resolve this problem by introducing an approximation to the Wasserstein distance, followed by establishing a statistical manifold via the so induced distance, which exhibits the desirable property of being complete within mixture models. The subsequent optimisation along this manifold intrinsically satisfies the probability constraint and ensures that the solution resides in the same mixture models. More importantly, minimising the induced distance is shown to be truly reducing the discrepancy between two mixture models, unlike most existing solutions which are based on the minimisation of the Euclidean distance from the optimal parameters. This ensures fast and stable convergence in optimisation. By realising that the existing Riemannian adaptive algorithms only make sense in updating vector parameters, we further develop a novel accelerated stochastic gradient descent method for updating the positive definite matrices. 

In this way, our proposed framework makes it possible to incorporate a broad family of distributions, $\rho(\bm{\theta})$, including an important class of multivariate analysis techniques called elliptical distributions \cite{fang2018symmetric} and to further investigate the mixture family termed the elliptical mixture model (EMM). We therefore provide computable and identifiable EMMs in a unified way, which demonstrates that EMMs are quite general and flexible and include the GMMs as special cases \cite{fang2018symmetric}.

Overall, this paper proposes a complete and efficient framework for solving general EMM problems, by establishing a statistical manifold under an approximate Wasserstein distance which promotes stability and efficiency, together with an adaptive stochastic gradient algorithm to further accelerate the optimisation\footnote{The code of this paper is available at https://github.com/ShengxiLi/wass$\_$emm}. Compared to the existing literature on mixture problems, the proposed solution achieves consistently superior performance not only in the GMM problems but also for general EMM problems. Our contributions can be summarised as follows:
\begin{itemize}
	\item A unified framework for dealing with  computable and identifiable EMMs, which introduces a rich choice of candidates for flexible finite mixture models.
	\item Establishment of the statistical manifold through the proposed approximate Wasserstein distance, which provides explicit and complete operations within the manifold.
	\item An Adaptive accelerated Riemannian gradient descent algorithm on the established manifold, to improve the optimisation and accelerate convergence.
\end{itemize}

\section{Computable and Identifiable EMMs}
Elliptical distributions include a wide range of standard distributions, and it therefore comes as no surprise that a unified summary of computable candidates as components in the EMMs is a prerequisite to problem definition and subsequent solutions. A classical summary can be found in Chapter 3 in \cite{fang2018symmetric}; however, despite progress this framework is not general enough as various elliptical distributions are still missing, and more importantly, it involves complicated representations for each type of elliptical distributions. The existing literature also employs different notations and formulations of particular distributions, which may lead to confusion. To this end, we provide a simple and unified framework for summarising the existing computable elliptical distributions via the stochastic representation, which can then be used to constitute flexible and identifiable EMMs.


\subsection{Preliminaries on Elliptical Distributions}
A random variable, $\bm{\mathcal{X}}\in \mathbb{R}^m$, is said to exhibit an elliptical distribution if and only if it admits the following stochastic representation \cite{fang2018symmetric},
\begin{equation}\label{RESsto}
\bm{\mathcal{X}} =^d \bm \mu + \mathcal{R}\mathbf{\Lambda}\bm{\mathcal{S}},
\end{equation}
where $\mathcal{R} \in \mathbb{R}^+$ is a non-negative real scalar random variable which models tail properties of the elliptical distribution; $\bm{\mathcal{S}} \in \mathbb{S}^{(m'-1)}$ is a random vector which is uniformly distributed on a unit spherical surface\footnote{The term $\mathbb{S}^{m'-1}$ is defined as $\mathbb{S}^{m'-1}:=\{\mathbf{x}\in\mathbb{R}^{m'}: \mathbf{x}^T\mathbf{x}=1\}$.} with the pdf within the class of $2\pi^{\nicefrac{-m'}{2}}\Gamma(\nicefrac{m'}{2})$; $\bm \mu \in \mathbb{R}^m$ is a mean (location) vector, while $\mathbf{\Lambda} \in \mathbb{R}^{m \times m'}$ is a matrix that transforms $\bm{\mathcal{S}}$ from a sphere to an ellipse, and  ``$=^d$" designates ``the same distribution''. 
For a comprehensive review of elliptical distributions, we refer to \cite{fang2018symmetric}. 

When $m' = m$, that is, for a non-singular scatter matrix $\mathbf{\Sigma} = \mathbf{\Lambda}\mathbf{\Lambda}^T$, the pdf for elliptical distributions does exist and has the following form
\begin{equation}\label{RES}
p(\mathbf{x}) \!=\!\underbrace{2\pi^{-\frac{m}{2}}\Gamma(\frac{m}{2})}_{c_m} \mathrm{det}(\mathbf{\Sigma})^{-\frac{1}{2}}g(\underbrace{(\mathbf{x} - \bm{\mu})^T \mathbf{\Sigma}^{-1} (\mathbf{x} - \bm{\mu})}_{t}).
\end{equation}
In \eqref{RES}, the term $c_{m}$ serves as a normalisation term and relates solely to $m$. We denote the Mahalanobis distance $(\mathbf{x} - \bm{\mu})^T \mathbf{\Sigma}^{-1} (\mathbf{x} - \bm{\mu})$ by $t$. The density generator, $g(t)$, can be explicitly expressed as $t^{-\nicefrac{(m-1)}{2}}p_\mathcal{R}(\sqrt{t})$, where $t>0$ and $p_\mathcal{R}(t)$ denotes the pdf of $\mathcal{R}$. Thus, $\mathcal{R}$, or equivalently\footnote{The term $\mathcal{R}^2$ is frequently used in practice because $\mathcal{R}^2 =^d (\mathbf{x} - \bm{\mu})^T \mathbf{\Sigma}^{-1} (\mathbf{x} - \bm{\mu})$.} $\mathcal{R}^2$, fully characterises $g(\cdot)$, i.e., the type of elliptical distributions. For example, when $\mathcal{R}^2=^d {\chi_m^2}$ ($\chi_m^2$ denotes the chi-squared distribution of dimension $m$), then in \eqref{RES}, $g(t) \propto \mathrm{exp}(-\nicefrac{t}{2})$, which formulates the multivariate Gaussian distribution. Therefore, the elliptical distribution can be fully characterised by $\bm{\mu}$, $\mathbf{\Sigma}$ and $\mathcal{R}$.
For simplicity, the elliptical distribution in \eqref{RES} will be denoted by $\bm{\mathcal{X}}\sim\mathcal{E}(\mathbf{x};\bm{\mu}, \mathbf{\Sigma}, \mathcal{R})$, where $\mathcal{E}(\mathbf{x};\bm{\mu}, \mathbf{\Sigma}, \mathcal{R}) = c_m\mathrm{det}(\mathbf{\Sigma})^{-1/2}g(t)$ of \eqref{RES}. 


\subsection{Computable and Identifiable EMMs}\label{secsummaryEMM}
Due to the fact that the $\mathcal{R}^2$ decides the type of elliptical distributions, we here provide a unified summary of elliptical distributions in Table \ref{typicalRES}; this is achieved through stochastic representations in \eqref{RESsto}. This makes it possible to avoid complicated formulations, and to instead classify different categories simply through several typical distributions of $\mathcal{R}^2$, which also allows for simple and intuitive sample generations for elliptical distributions. The proof for expressions in this table is provided in Appendix. Uniquely, this further clarifies the commonalities between the members of the elliptical family of distributions. More importantly, constructing an EMM with the candidates in Table \ref{typicalRES} can be easily proved to be identifiable based on Theorem 2 in \cite{holzmann2006identifiability}. It is thus convenient and safe to establish a well-defined EMM by the candidates summarised in Table \ref{typicalRES}.

\begin{table*}[!htb]
	\centering
	\caption{$\mathcal{R}^2 \leftrightarrow$ Computable elliptical distributions}\label{38test}	
		{\scriptsize{\begin{tabular}{|C{1.5cm}| c| l| c|}
				\hline\hline
				{Types} & \multicolumn{1}{c}{$\mathcal{R}^2$} & \multicolumn{1}{l|}{$\leftrightarrow ~~~~~~~~~~~~~~~~~~~~~c_m \cdot g(t)$}  & Typical Multivariate Dist. \\
				\hline
				{\multirow{4}{1.5cm}{\centering \textbf{Kotz Type} \\ Ref[a]}}& \multirow{2}{*}{$\mathcal{R}^2 =^d \mathcal{G}^{1/s},$} & 	\multirow{4}*{$= \left(\frac{\Gamma(\nicefrac{m}{2})sb^{\nicefrac{(2a+m-2)}{(2s)}}}{\Gamma(\nicefrac{(2a+m-2)}{2s})\pi^{\nicefrac{m}{2}}}\right)t^{a-1}\mathrm{exp}(-bt^s)$} & Gamma: $s=1$ \\
				&		&		 &  Weibull: $a = s$\\
				&$\mathcal{G}\sim\mathrm{Ga}(\frac{2a+m-2}{2s}, b)$ & &	Generalised Gaussian: $a = 1$\\
				& $a>1\!-\!\frac{m}{2},b,s\!>\!0$ &		& Gaussian: $a = 1,~ s = 1,~ b = \frac{1}{2}$	\\
				\hline
				{\multirow{4}{1.5cm}{\centering \textbf{Scale Mixture of Normals}}} & \multirow{2}*{Pearson Type VII Ref[a]} & \multirow{4}*{$= \left(\frac{(\pi v)^{-m/2}\Gamma(s)}{\Gamma(s - m/2)}\right)(1+\nicefrac{t}{v})^{-s}$}  & \multirow{2}*{$T$-dist.: $s = \frac{m+v}{2}$}\\
				& & & \\
				&	$\mathcal{K}^{-1}\!\!\sim\!\mathrm{Ga}(s\! -\! \frac{m}{2}\!, \!\frac{v}{2})$,	& &\multirow{2}*{Cauchy: $v = 1$, $s = \frac{m+1}{2}$}\\
				& $v>0$, $s>\nicefrac{m}{2}$& & \\\cline{2-4}
				\multirow{6}{1.5cm}{\centering $\mathcal{R}^2 =^d \mathcal{G}\cdot\mathcal{K}$, $\mathcal{G}\sim\mathrm{Ga}(\frac{m}{2},\frac{1}{2})$, $\mathcal{K}$ has different dist.}& \centering Hyperbolic Type Ref[b] & \multirow{3}*{ $=\left(\frac{({v}/{a})^{\lambda/2}}{(2\pi)^{m/2}\mathrm{BeK}_\lambda(\sqrt{av})}\right)\frac{\mathrm{BeK}_{(\lambda-m/2)}(\sqrt{av + vt})}{(\sqrt{\nicefrac{a}{v}+\nicefrac{t}{v}})^{m/2-\lambda}}$ }  & Inverse-Gaussian: $\lambda = -1/2$\\		
				&$\mathcal{K}\sim \mathrm{GIG}(v,a,\lambda)$ & &$K$-dist. \cite{ollila2012complex}: $a\rightarrow0, ~\lambda>0$\\	
				& $v,a>0,~\lambda\in \mathbb{R}$ & & Laplace: $a\rightarrow0,~\lambda=1,~v=2$\\	 \cline{2-4}
				&  Other Types Ref[c]  & \multirow{2}*{$=\nicefrac{\mathrm{exp}(-t)}{(1+\mathrm{exp}(-t))^2}$}  & \multirow{2}*{Logistic} \\
				& $\sqrt{\mathcal{K}}\sim \partial \mathrm{Kov}(\frac{\mathcal{K}}{2})/\partial \mathcal{K}$ & &\\ \cline{2-4}
				& $\mathcal{K}\sim\mathrm{S\alpha S}({\frac{a}{2}})$, $a\in(0,2)$& $\propto \mathrm{S\alpha S}(a) $ &$\alpha$-stable\\\hline
				\multirow{2}{1.5cm}{\centering \textbf{Pearson Type II} } & \multirow{2}*{$\mathcal{R}^2\sim\mathrm{Beta}(m/2, s)$, $s>1$} & \multirow{2}*{$= \left(\frac{\Gamma(m/2+s)}{\pi^{m/2}\Gamma(s)}\right)(1-t)^{s-1}$, $t\in[0,1]$} & \multirow{2}*{Ref[a]}\\
				&  &  &\\\hline \hline
				{\textbf{Notations:}}&\multicolumn{3}{l|}{{$a,b,s,v,\lambda,\alpha$ are adjustable parameters for different types of dist.; $\mathcal{G}$ and $\mathcal{K}$ are random variables related to $\mathcal{R}^2$; $m$ is the dimension.} }\\
				&\multicolumn{3}{l|}{{Gamma dist.: $\mathrm{Ga}(x,y) = y^x t^{x-1} \mathrm{exp}(-yt)/\Gamma(x)$; Inverse Gaussian dist.: $\mathrm{GIG}(x,y,z)=\frac{(x/y)^{z/2}}{2\mathrm{BeK}_z(\sqrt{xy})}t^{z-1}\mathrm{exp}(-\frac{xt^2+y}{2t})$;}}\\
				&\multicolumn{3}{l|}{{Kolmogorov-Smirnov dist.: $\mathrm{Kov}(x) = 1-2\sum_{n=1}^{\infty}(-1)^{n+1}\mathrm{exp}(-2n^2x^2)$; Beta dist.: $\frac{\Gamma(x+y)}{\Gamma(x)\Gamma(y)}t^{x-1}(1-t)^{y-1}$;}}\\
				&\multicolumn{3}{l|}{{$\mathrm{S\alpha S}(a)$: the symmetric $\alpha$-stable dist. with index $a$; $\mathrm{BeK}_x(y)$: the Bessel function of the third kind; $\Gamma(x)$: the Gamma function.}}\\\hline
				{\textbf{References:}}&\multicolumn{3}{l|}{[a]: \cite{fang2018symmetric}; [b]: \cite{browne2015mixture}; [c]: \cite{andrews1974scale,stefanski1991normal}}\\\hline
		\end{tabular}}}\label{typicalRES}
\end{table*}

\subsection{Problem Statement and Notations}\label{notations}
For generality, we assume that there are $k$ mixtures in the EMM model, latent variables $\mathcal{Z}_i \in \{0,1\}$ are binary, and the probability of choosing the $i$-th mixture is denoted by $p(\mathcal{Z}_i = 1) = \pi_i$, so that $\sum_{i = 1}^k \mathcal{Z}_i = 1$ and $\sum_{i = 1}^k \pi_i = 1$. If we use the random variable $\bm{\mathcal{Y}}$ to denote the EMM, then $\bm{\mathcal{Y}}=^d\sum_{i=1}^{k}\mathcal{Z}_{i}\bm{\mathcal{X}}_{i} \sim \sum_{i=1}^{k}\pi_{i} \mathcal{E}(\mathbf{x}|\bm{\mu}_{i}, \mathbf{\Sigma}_{i}, \mathcal{R})$, and the pdf of $\bm{\mathcal{Y}}$ can be expressed as
\begin{equation}\label{EMMpdf}
p(\mathbf{y}) = \sum_{i=1}^{k} \pi_i c_m \mathrm{det}(\mathbf{\Sigma}_i)^{-\frac{1}{2}}g((\mathbf{y} - \bm{\mu}_i)^T \mathbf{\Sigma}_i^{-1} (\mathbf{y} - \bm{\mu}_i)),
\end{equation}
where the density generator $g(\cdot)$ can be chosen flexibly from Table \ref{typicalRES}, for which the identifiability is ensured \cite{holzmann2006identifiability}. Thus, without ambiguity in the context, we shall  denote the EMM $\bm{\mathcal{Y}}$ through the pdf of \eqref{EMMpdf}, that is, as $\bm{\mathcal{Y}}_{\bm{\theta}}$, where $\bm{\theta} = \{\bm{\pi}, \bm{\mu}_i, \mathbf{\Sigma}_i\}$ ($\bm{\pi} = [\pi_1, \pi_2, \ldots, \pi_k]^T$ and $i = 1,2,\ldots,k$).


\section{Statistical Manifold towards EMMs}\label{secstatis}
As the Wasserstein distance possesses a Riemannian structure \cite{kolouri2017optimal}, it is then natural to treat each EMM as a ``point'' in the manifold, whereby the metric is defined by the Hessian of the Wasserstein distance. 
However, as pointed out by Chen \cite{chen2018natural}, in the Wasserstein space, the geodesic between two points, i.e., two GMMs, does not necessarily belong to the GMM of the same type. A way to solve this problem is to pull back from the whole density space to the parametric space \cite{chen2018natural}. This treatment, nevertheless, is still a toy solution due to two main considerations. The first deficiency is that the metric does not have a closed-form representation, and needs to be numerically obtained for every possible value in advance before the optimisation. This is highly prohibitive, especially for multivariate cases due to the curse of dimensionality. Besides computational intractability, operations such as the exponential mapping and the vector transport in the manifold cannot be well defined as there is typically a second-order differential equation involved to obtain those operations.

We thus propose an approximate Wasserstein distance between two EMMs as a means to define a well-behaved manifold for the EMM problems; this is achieved by the property that the Wasserstein distance of two elliptical distributions is completely and explicitly defined. We then provide the Riemannian metric for the EMM problems according to the defined distance.

\subsection{Approximate Wasserstein Distance between EMMs}
We now focus on the distance between two EMMs $\bm{\mathcal{Y}}_1$ and $\bm{\mathcal{Y}}_2$,
and propose an approximate Wasserstein distance by treating each distribution within an EMM as a ``super-point'' and defining a transport-like distance between those ``super-points''. A rigorous definition is given as follows. 

\begin{definition}\label{def}
	Given two EMMs $\bm{\mathcal{Y}}_1$ and $\bm{\mathcal{Y}}_2$, a discrepancy measure is defined as
	\begin{equation}\label{discdef}
	\begin{aligned}
	d_U(\bm{\mathcal{Y}}_1,\bm{\mathcal{Y}}_2)\! =\! \min_{\gamma(i,j)}&\big( \sum_{i,j}\frac{\gamma(i,j)}{k}d^2_W(\bm{\mathcal{X}}_{i,1},\bm{\mathcal{X}}_{j,2})\\
	&+\arccos(\sum_{i,j}\gamma(i,j)\sqrt{\pi_{i,1}\pi_{j,2}})\big),
	\end{aligned}
	\end{equation}
	where $d^2_W(\bm{\mathcal{X}}_{i,1},\bm{\mathcal{X}}_{j,2})$ is the Wasserstein distance between the elliptical distributions $\bm{\mathcal{X}}_{i,1}$ and $\bm{\mathcal{X}}_{j,2}$. $\gamma(i,j)$ is binary $\in \{0,1\}$; for each $i$ and $j$, $\gamma(i,j)$ satisfies $\sum_{i=1}^{k}\gamma(i,j) = 1$ and $\sum_{j=1}^{k}\gamma(i,j) = 1$.
\end{definition}
\begin{theorem}\label{Theorem1}
	Given two EMMs $\bm{\mathcal{Y}}_1$ and $\bm{\mathcal{Y}}_2$,  the discrepancy measure $d_U(\bm{\mathcal{Y}}_1,\bm{\mathcal{Y}}_2)$ defines a distance.
\end{theorem}
\begin{proof}
	Please see Appendix. 
\end{proof}

We provide our intuitions of defining the distance of \eqref{discdef}.  The term $\arccos(\sum_{i,j}\gamma(i,j)\sqrt{\pi_{i,1}\pi_{j,2}})$ intrinsically relates to the probability constraint of $\sum_{i = 1}^k \pi_{i,1} = \sum_{j = 1}^k \pi_{j,2} = 1$, while $\gamma(i,j)$ operates as a bijection between mixture components in $\bm{\mathcal{Y}}_1$ and $\bm{\mathcal{Y}}_2$.  $\sum_{i,j}\frac{\gamma(i,j)}{k}d^2_W(\bm{\mathcal{X}}_{i,1},\bm{\mathcal{X}}_{j,2})$ can thus be regarded as a discrete transport between $k$ uniformly distributed ``super-points'' $\bm{\mathcal{X}}_{i,1}$ and $\bm{\mathcal{X}}_{j,2}$, for  which their cost is defined as 
$d^2_W(\bm{\mathcal{X}}_{i,1},\bm{\mathcal{X}}_{j,2})$. We also noticed a computational distance proposed by \cite{chen2018optimal}. Our metric can be further seen as a restriction (or upper bound) of the Chen's metric because we operate via a one-to-one transport whilst the Chen's metric admits arbitrary transport plans. However, the most advantageous property of our metric is that it defines an explicit manifold over EMMs, which will be introduced shortly.

More importantly, $d_U(\bm{\mathcal{Y}}_1,\bm{\mathcal{Y}}_2)$ comprehensively reflects the discrepancy between $\bm{\mathcal{Y}}_1$ and $\bm{\mathcal{Y}}_2$ via a summation operation; this means $d_U(\bm{\mathcal{Y}}_1,\bm{\mathcal{Y}}_2)=0$ if and only if both the difference between mixture components $\bm{\mathcal{X}}_{i,1}$ and $\bm{\mathcal{X}}_{j,2}$ and the difference between latent variables $\pi_{i,1}$ and $\pi_{j,2}$ equal to 0. The following lemma further proves that for balanced EMMs, $d_U(\bm{\mathcal{Y}}_1,\bm{\mathcal{Y}}_2)$ is an upper bound of the Wasserstein distance $d^2_W(\bm{\mathcal{Y}}_1,\bm{\mathcal{Y}}_2)$.

\begin{lemma}\label{Lemma1}
	Given two balanced EMMs $\bm{\mathcal{Y}}_1$ and $\bm{\mathcal{Y}}_2$ (i.e., $\pi_{i,1} = \pi_{j,2} = \nicefrac{1}{k}$ for all $i,j$), $d_U(\bm{\mathcal{Y}}_1,\bm{\mathcal{Y}}_2)$ is an upper bound of the Wasserstein distance:
	\begin{equation}
	d^2_W(\bm{\mathcal{Y}}_1,\bm{\mathcal{Y}}_2) \leq d_U(\bm{\mathcal{Y}}_1,\bm{\mathcal{Y}}_2).
	\end{equation}
	The equality holds when $k=1$.	
\end{lemma}
\begin{proof}
	Please see Appendix, where the tightness of the upper bound is also analysed.
\end{proof}



\subsection{Statistical Manifold for EMM Problems}
Before introducing the statistical manifold for EMMs, we would like to give credit to several most recent works on Gaussian distributions and elliptical distributions, which lay a basis of our metric proposed in this section. We would like to mention \cite{knott1984optimal} of the Wasserstein distance of Gaussian measures and \cite{muzellec2018generalizing} for the elliptical distributions. Most recently, the Riemannian manifold for Gaussian distributions has been established in an explicit form \cite{takatsu2011wasserstein,malago2018wasserstein}.
Then, on the basis of the approximated Wasserstein distance, we provide the Riemannian metric of the EMM problems by calculating the Hessian of the proposed distance in Definition \ref{def} as follows. 
\begin{lemma}\label{manifold}
	The approximate Wasserstein distance $d_U(\bm{\mathcal{Y}}_1,\bm{\mathcal{Y}}_2)$ represents an explicit Riemannian metric in the parametric space, and the corresponding Riemannian manifold is a product manifold of $ \mathbb{R}^k\times\prod_{i=1}^k(\mathbb{R}^m\times\mathbb{P})$, where $\mathbb{P}$ is the $m\times m$ positive definite manifold:
	
	1) The manifold for the square root of $\pi_i$, i.e., $[\sqrt{\pi}_1, \sqrt{\pi}_2, \cdots, \sqrt{\pi}_k]^T$ is a sphere manifold of $\mathbb{R}^k$.
	
	2) The manifold for $\bm{\mu}_i$ is the Euclidean space of $\mathbb{R}^m$.	
	
	3) The manifold for $\mathbf{\Sigma}_i$, i.e., $\mathbb{P}$, is defined by
	\begin{equation}\label{psdmanifold}
	ds^2 = \frac{\mathbb{E}[\mathcal{R}^2]}{m} (\mathbb{L}_{\mathbf{\Sigma}_i}[d\mathbf{\Sigma}])\mathbf{\Sigma}_i(\mathbb{L}_{\mathbf{\Sigma}_i}[d\mathbf{\Sigma}]).
	\end{equation} 
	In \eqref{psdmanifold}, $\mathbb{L}_\mathbf{A}[\mathbf{C}] = \mathbf{B}$ is a Lyapunov operator: $\mathbf{A}\mathbf{B}+\mathbf{B}\mathbf{A} = \mathbf{C}$, where $\mathbf{A}, \mathbf{B}, \mathbf{C}\in\mathbb{P}$. More importantly, the sectional curvature is non-negative ($=\nicefrac{m}{\mathbb{E}[\mathcal{R}^2]} {k}_{\mathrm{G}}$ where ${k}_{\mathrm{G}}\geq0$ is the sectional curvature for Gaussian cases \cite{takatsu2011wasserstein}). Recall that $\mathcal{R}$ is defined in \eqref{RESsto}. 
\end{lemma}
\begin{proof}
	Please see Appendix. 
\end{proof}
\begin{remark}
	The metric in \eqref{psdmanifold} provides a manifold for positive definite matrices. Compared to the best known manifold belonging to the Hadamard manifold (non-positive sectional curvature) \cite{sra2015conic}, the newly developed manifold provides an example of non-negative manifolds. It is actually further stated that an Alexandrov space has the non-negative curvature iff it is a Wasserstein space \cite{sturm2006geometry}. The established metric thus provides a desirable reflection on the curvature when dealing with Wasserstein related EMM problems.
\end{remark}

\begin{table*}[!t]
	\centering
	\caption{Basic operations of the manifold in Lemma \ref{manifold}}
	\scriptsize{\begin{tabular}{c|cc}
			\hline
			\multicolumn{3}{l}{For the $h$-th iteration } \\ \hline
			&  $\nabla_U(\cdot)$   &  	$\mathrm{Exp}(-\alpha\nabla_U(\cdot))$     \\
			$\sqrt{\bm{\pi}}^h$ &  $\nabla_E(\sqrt{\bm{\pi}}^h) - (\sqrt{\bm{\pi}}^{h})^T\nabla_E(\sqrt{\bm{\pi}}^h)\cdot \sqrt{\bm{\pi}}^{h}$   &  	$\cos(||\alpha\nabla_U(\sqrt{\bm{\pi}}^h)||_2)\sqrt{\bm{\pi}}^h - \frac{\sin(||\alpha\nabla_U(\sqrt{\bm{\pi}}^h)||_2)}{||\nabla_U(\sqrt{\bm{\pi}}^h)||_2}\nabla_U(\sqrt{\bm{\pi}}^h)$     \\
			$\bm{\mu}_i^h$ &  $\nabla_E(\bm{\mu}_i^h)$    &  $\bm{\mu}_i^h-\alpha\nabla_U(\bm{\mu}_i^h)$    \\
			$\mathbf{\Sigma}_i^h$ & $\nabla_E(\mathbf{\Sigma}_i^h)\mathbf{\Sigma}_i^h + \mathbf{\Sigma}_i^h\nabla_E(\mathbf{\Sigma}_i^h)$  &   $(\mathbb{L}_{\mathbf{\Sigma}_i^h}[-\alpha\nabla_U(\mathbf{\Sigma}_i^h)] + I)\mathbf{\Sigma}_i^h(\mathbb{L}_{\mathbf{\Sigma}_i^h}[-\alpha\nabla_U(\mathbf{\Sigma}_i^h)] + I)$  \\ \hline  
			\multicolumn{3}{l}{For the Radon transform $\mathbf{p}\in \mathbb{S}^{m-1}$, specified in the EMM problems:} \\ \hline
			{$\nabla_E(\sqrt{\bm{\pi}}^h)$} &\multicolumn{2}{c}{The $i$-th dimension of $\nabla_E(\sqrt{\bm{\pi}}^h)$ is $2\int_\mathbb{R} c_m \phi(y) \sqrt{{\pi}_i}^h (\mathbf{p}^T\mathbf{\Sigma}^h_i\mathbf{p})^{-\nicefrac{1}{2}}g(\frac{(y - \mathbf{p}^T\bm{\mu}_i)^2}{\mathbf{p}^T\mathbf{\Sigma}^h_i\mathbf{p}})dy$}\\
			{$\nabla_E(\bm{\mu}_i^h)$} &\multicolumn{2}{c}{$=\left(-2\int_\mathbb{R}c_m\phi(y)\pi_{i}^h(\mathbf{p}^T\mathbf{\Sigma}^h_i\mathbf{p})^{-\nicefrac{3}{2}}g'(\frac{(y - \mathbf{p}^T\bm{\mu}_i)^2}{\mathbf{p}^T\mathbf{\Sigma}^h_i\mathbf{p}})(y-\mathbf{p}^T\bm{\mu}_i^h)dy\right)\mathbf{p}$}\\
			{$\nabla_E(\mathbf{\Sigma}_i^h)$} &\multicolumn{2}{c}{$=\left(-\int_\mathbb{R}c_m\phi(y)\pi_{i}^h(\mathbf{p}^T\mathbf{\Sigma}^h_i\mathbf{p})^{-\nicefrac{3}{2}}\big(\frac{1}{2}g(\frac{(y - \mathbf{p}^T\bm{\mu}_i^h)^2}{\mathbf{p}^T\mathbf{\Sigma}^h_i\mathbf{p}}) + g'(\frac{(y - \mathbf{p}^T\bm{\mu}_i^h)^2}{\mathbf{p}^T\mathbf{\Sigma}^h_i\mathbf{p}})\frac{(y - \mathbf{p}^T\bm{\mu}_i^h)^2}{\mathbf{p}^T\mathbf{\Sigma}^h_i\mathbf{p}}\big)dy\right)\mathbf{p}\mathbf{p}^T$}\\\hline
			\multicolumn{3}{l}{$\phi(y)$ is the Kantorovich potential \cite{chen2018natural}. $\nabla_E(\cdot)$ denotes the Euclidean gradient with regard to $\sqrt{\bm{\pi}}^h$, $\bm{\mu}_i^h$ and $\mathbf{\Sigma}_i^h$.} \\\hline
		\end{tabular}\label{basciops}
	}
\end{table*}

\section{Adaptively Accelerated Optimisation}
By virtue of optimising on a statistical manifold, the probability constraint can be satisfied automatically; this allows us to incorporate various numerical algorithms when solving the EMM problems. More specifically, we first show that the constrained minimisation problem can be transformed to an unconstrained one when restricted on the statistical manifold, which results in the ``vanilla'' gradient descent on the manifold. We then propose an adaptively accelerated solver.

\subsection{Vanilla Gradient Descent on Statistical Manifold}
Similar to \cite{martens2014new}, for gradient descent methods, the way for gradient descent in our work is given by,
\begin{equation}\label{opt}
\begin{aligned}
\bm{\mathcal{Y}}_{\bm{\theta}+\Delta\bm{\theta}^*}\!\! =& \argmin_{\bm{\mathcal{Y}}_{\bm{\theta}+\Delta\bm{\theta}}}d_{SW}(\bm{\mathcal{Y}}_{\bm{\theta}+\Delta\bm{\theta}},\bm{\mathcal{Y}}^*)\!+\!\frac{d_U(\bm{\mathcal{Y}}_{\bm{\theta}},\bm{\mathcal{Y}}_{\bm{\theta}+\Delta\bm{\theta}})}{c^2},\\
&~~~~~~~~~~~~~~\rightarrow \Delta\bm{\theta}^* = \mathrm{Exp}(-\alpha\nabla_U(\bm{\theta}))
\end{aligned}
\end{equation}
where $\Delta\bm{\theta}$ is the step size for the next iteration and $c^2$ denotes a sphere of all realisable distributions, which ensures searching for optimal $\Delta\bm{\theta}^*$ without being slowed down by the curvature; $d_{SW}(\bm{\mathcal{Y}}_{\bm{\theta}+\Delta\bm{\theta}},\bm{\mathcal{Y}}^*)$ denotes the sliced Wasserstein distance \cite{rabin2011wasserstein} between the parametric mixture model $\bm{\mathcal{Y}}_{\bm{\theta}+\Delta\bm{\theta}}$ and the observed samples $\bm{\mathcal{Y}}^*$. By approximating $d_U(\bm{\mathcal{Y}}_{\bm{\theta}},\bm{\mathcal{Y}}_{\bm{\theta}+\Delta\bm{\theta}})$ with its Hessian term (i.e., the inner product on the tangent space), we naturally obtain an unconstrained Riemannian gradient descent of the second row of \eqref{opt}, where $\nabla_U(\bm{\theta})$ is the Riemannian gradient on the statistical manifold defined by $d_U(\cdot,\cdot)$; $\alpha$ is the stepsize of each iteration; $\mathrm{Exp}(-\alpha\nabla_U(\bm{\theta}))$, called exponential mapping, projects the step movement $-\alpha\nabla_U(\bm{\theta})$ from the tangent space to the statistical manifold along geodesics \cite{absil2009optimization}, which ensures the probability constraint. On the other hand, when the proposed distance $d_U(\bm{\mathcal{Y}}_{\bm{\theta}},\bm{\mathcal{Y}}_{\bm{\theta}+\Delta\bm{\theta}})$ is changed to the Euclidean distance, \eqref{opt} is the trivial gradient descent \cite{kolouri2018sliced}. However, as mentioned above, this gradient descent does not satisfy the probability constraint nor does it reflect the probability space curvature, which leads to inefficient and unstable optimisation. 

In \eqref{opt}, the sliced Wasserstein distance \cite{rabin2011wasserstein} provides a feasible solution in solving the semi-discrete Wasserstein problem, i.e., $\bm{\mathcal{Y}}_{\bm{\theta}+\Delta\bm{\theta}}$ is continuous and $\bm{\mathcal{Y}}^*$ is the sum of Dirac masses \cite{levy2015numerical}. It uses unit random projections to turn the original problem to a one-dimensional Wasserstein problem via the Radon transform \cite{rabin2011wasserstein}, so that a closed-form solution can be obtained\footnote{It needs to be pointed out that other approximate distances (e.g., entropy relaxed Wasserstein distances) for the semi-discrete problem can also be seamlessly adopted in \eqref{opt} of our work.}. 
More importantly, as the sliced Wasserstein distance is composed of a set of random projections \cite{rabin2011wasserstein}, it is then natural to implement a stochastic gradient descent on the Riemannian manifold for each random projection (denoted as $\mathbf{p}\in\mathbb{S}^{m-1}$). The random projection process also allows for parallel gradient descent.
We provide the basic operations of problem in \eqref{opt} in Table \ref{basciops}, while the details are provided in Appendix.

\begin{table}[!t]
	\centering
	\resizebox{\columnwidth}{!}{\begin{tabular}{ccccl}
		\hline
		\multicolumn{5}{l}{\textbf{Alg. 1}: Riemannian adaptively accelerated manifold  optimisation} \\ \hline
		\multicolumn{5}{l}{\textbf{Input}: $n$ observed samples $\mathbf{y}_1, \mathbf{y}_2, \ldots, \mathbf{y}_n$; stepsize $\{\alpha^h\}_{h=1}^H$; } \\ 
		& & \multicolumn{3}{l}{hyper-parameters $\{\beta_{1}^h\}_{h=1}^{H}$ and $\beta_{2}$} \\ 		
		\multicolumn{5}{l}{\textbf{Initialise}: $1^{st}$-order moment $\{\mathbf{u}_i^0\}_{i=1}^k$; $2^{nd}$-order moment $\{\mathbf{v}_i^0\}_{i=1}^k$;} \\
		& \multicolumn{4}{l}{\textbf{for}: $h=1$ to $H$ \textbf{do}} \\	
		& &\multicolumn{3}{l}{Random projection: $\mathbf{p} \in \mathbb{S}^{(m-1)}$} \\		
		& & \multicolumn{3}{l}{Update $\sqrt{\bm{\pi}}^{h+1} = \mathrm{Exp}(-\alpha\nabla_U(\sqrt{\bm{\pi}}^h))$}\\
		& & \multicolumn{3}{l}{\textbf{for}: $i=1$ to $k$ \textbf{do}} \\	
		& & & \multicolumn{2}{l}{Update $\bm{\mu}_i^{h+1} = \mathrm{Exp}(-\alpha\nabla_U(\bm{\mu}_i^h))$}\\
		& & & & Update $\mathbf{\Sigma}_i^h$ by the \textit{Dadam}: \\
		& & & & $\mathbf{u}_i^h = \beta_{1}^h\varphi_{\mathbf{\Sigma}_i^{h-1}\rightarrow\mathbf{\Sigma}_i^h}(\mathbf{u}_i^{h-1}) + (1-\beta_{1}^h)\nabla_U(\mathbf{\Sigma}_i^h)$\\
		& & & & $\mathbf{v}_i^h = \beta_{2}\mathbf{v}_i^{h-1} + (1-\beta_{2})\nabla_E(\mathbf{\Sigma}_i^h)\nabla_E(\mathbf{\Sigma}_i^h)^T$\\
		& & & & $\mathrm{adp}_i^h = \max\{\mathbf{p}^T\mathbf{v}_i^h\mathbf{p},~\mathrm{adp}_i^{h-1}\}$\\
		& & & & $\mathbf{\Sigma}_i^{h+1} = \mathrm{Exp}( \nicefrac{-\alpha^h\mathbf{u}_i^h}{\sqrt{\mathrm{adp}_i^h}})$\\
		& & \multicolumn{3}{l}{\textbf{end for}}\\
		& \multicolumn{4}{l}{\textbf{end for}}\\
		\multicolumn{5}{l}{\textbf{Return}: $~\bm{\pi}^H,~ \{\bm{\mu}_i^H\}_{i=1}^k,~ \{\mathbf{\Sigma}_i^H\}_{i=1}^k$}\\\hline	
	\end{tabular}\label{algori}
}
\end{table}
\subsection{Adaptively Accelerated Algorithm}
On the basis of several accelerated Riemannian stochastic gradient descent methods which adopt the first-order moment information \cite{zhang2016first}, recently Becigneu and Ganea \cite{becigneul2018riemannian} further proposed a Riemannian adaptive method by employing the second-order moments. 
Although this adaptive stochastic gradient descent method \cite{becigneul2018riemannian} may be incorporated to our work to further improve the convergence, we argue that it only makes sense for updating vector parameters, i.e., $\bm{\pi}^h$ and $\bm{\mu}_i^h$, because in this case the second-order moment can be calculated element-wise by a decomposition into product manifolds \cite{becigneul2018riemannian}. This is similar to the \textit{Adam} algorithm in the Euclidean space \cite{kingma2014adam}. When updating the matrix parameter,  $\mathbf{\Sigma}_i^h$, however, it does not capture the second-order information properly due to the direct accumulation of the second-order moments over the whole matrix manifold; this slows down convergence in the optimisation. 

The key difficulty for the matrix case is that it is meaningless to accumulate second-order moments in an element-wise manner due to the structure within the matrix. We here view the second-order moments of matrices from another perspective, by realising that a positive definite matrix can be decomposed into a set of eigenvectors and the corresponding eigenvalues. The eigenvectors can be regarded as a set of projection directions and the eigenvalues are scalars that can be connected with the accumulation in those directions. Furthermore, we can see from Table \ref{basciops} that the Euclidean gradient $\nabla_E(\mathbf{\Sigma}_i^h)$ consists of a scalar weight multiplied by a rank-1 matrix characterised by the direction $\mathbf{p}\mathbf{p}^T$. 
Therefore, instead of the element-wise accumulation, we propose a direction-wise accumulation of second-order moments to adaptively adjust the stepsize when updating the matrix. The details of our algorithm in Algorithm 1, with our directional adaptive accelerated method (\textit{Dadam}) of updating $\mathbf{\Sigma}_i^h$ being achieved by $\mathbf{v}_i^h$ and $\mathrm{adp}_i^h$. 

Moreover, as there is no explicit parallel transport in the Wasserstein space, we propose a new means of vector transport  $\varphi_{\mathbf{\Sigma}_i^{h-1}\rightarrow\mathbf{\Sigma}_i^h}(\mathbf{u}_i^{h-1}) = \mathbb{L}_{\mathbf{\Sigma}_{i}^{h-1}}[\mathbf{u}_i^{h-1}]\mathbf{\Sigma}_{i}^{h} + \mathbf{\Sigma}_{i}^{h}\mathbb{L}_{\mathbf{\Sigma}_{i}^{h-1}}[\mathbf{u}_i^{h-1}]$ to accumulate the first-order moments.

As for the convergence analysis, because our goal is to minimise the Wasserstein distance, qualitatively, our manifold, defined by the Hessian of the approximate Wasserstein distance as in Lemma \ref{manifold}, implicitly involves the second-order information and accelerates the convergence. Furthermore, $\mathrm{adp}_i^h$ denotes the accumulation in the current projection $\mathbf{p}$, which is a scalar (for computational ease) and does not need any eigen-decomposition operations. This, on the one hand, avoids the high computational complexity of calculating any inverse of matrices, making our method the same computational complexity as that of \cite{becigneul2018riemannian} at each iteration. On the other hand, it also makes our convergence analysis similar to that of \cite{becigneul2018riemannian}, upon realising that the sectional curvature is automatically bounded from below; we omit the analysis here. 

\section{Experimental Results}
\begin{figure*}[!htb]
	\begin{center}
		\includegraphics[width=0.7\textwidth]{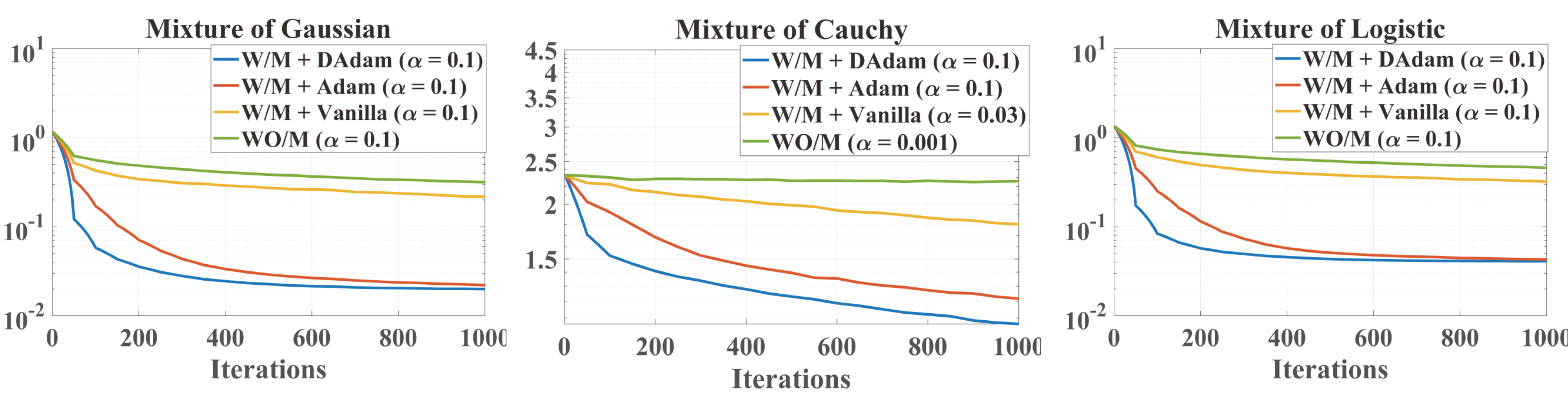}
	\end{center}
	\caption{The Wasserstein distance against the number of iterations for the four considered algorithms, averaged over the three types of datasets ($\{m, k\} = \{2,3\}, \{8,9\}, \{16,27\}$). The best learning rate is shown in the legend. \textit{1000} iterations are for illustration purpose. The computational time per iteration is discussed in Table \ref{random_trials}.}\label{convergence}
\end{figure*}
\begin{figure*}[!htb]
	\begin{center}
		\includegraphics[width=0.8\textwidth]{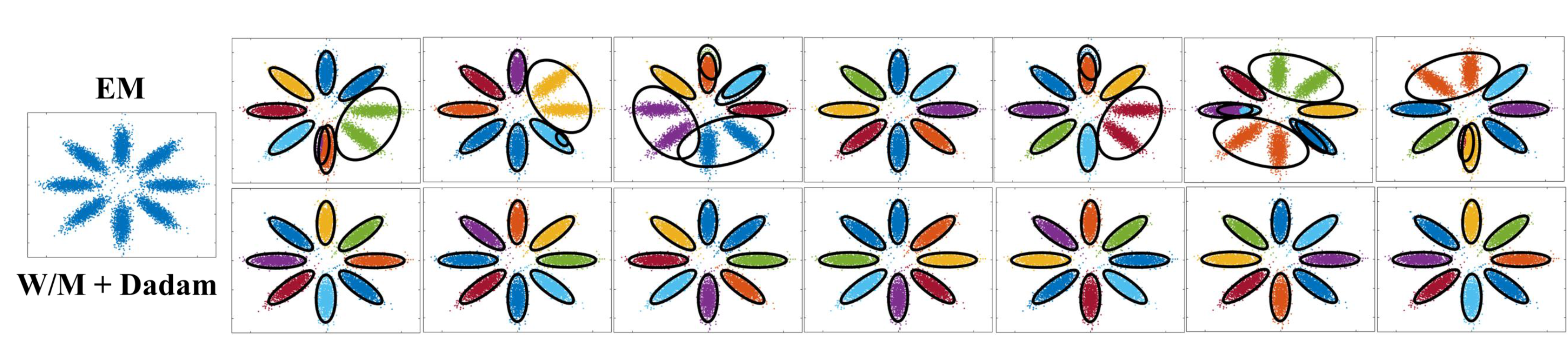}
	\end{center}
	\caption{The optimised GMMs via the EM and our (\textit{W/M+Dadam}) methods on a flower-shaped synthetic data ($N=10,000$) with random initialisations. Each cluster in the flower-shaped data is composed by Gaussian-distributed samples.}\label{random_trials}
\end{figure*}
\begin{table*}[!htb]
	\centering
	\caption{Comparisons among our, WO/M and EM methods for GMMs by varying $m$ and $k$.}
	\scriptsize{\begin{tabular}{|c|ccc|ccc|ccc|}
			\hline
			& \multicolumn{3}{c|}{$m=2,k=3$}            & \multicolumn{3}{c|}{$m=8,k=9$}             & \multicolumn{3}{c|}{$m=16,k=27$}           \\\hline
			& Wass            & NLL       & Time per Ite.     & Wass            & NLL        & Time per Ite.        &  Wass          & NLL     & Time per Ite.           \\\hline
			W/M+Dadam  & \textbf{0.01} $\pm$ \textbf{0.00} & \textbf{5.10} $\pm$ \textbf{0.00} & 4.35ms  & \textbf{0.03} $\pm$ \textbf{0.03} & \textbf{19.49} $\pm$ \textbf{0.03} & 9.83ms  & \textbf{2.39} $\pm$\textbf{ 0.33} & \textbf{44.92} $\pm$ \textbf{1.04}  &  32.51ms\\
			WO/M & 0.03 $\pm$ 0.00 & 5.12 $\pm$ 0.00  &  4.23ms & 4.26 $\pm$ 3.61 & 22.27 $\pm$ 1.69  & 9.64ms  & 5e3 $\pm$ 589   & Inf   &    29.77ms         \\
			EM   & 0.65 $\pm$ 1.06 & 5.30 $\pm$ 0.35  & 3.10ms  & 0.70 $\pm$ 0.46 & 19.83 $\pm$ 0.29  &  13.92ms & 2.51 $\pm$ 1.27 & 48.84 $\pm$ 1.24 & 215.00ms \\\hline 
		\end{tabular}\label{GMMcomp}}
\end{table*}

\begin{table*}[!htb]
	\centering
	\caption{Comparison performance of our, WO/M and IRA algorithms over the BSDS500 dataset via five metrics}
	\scriptsize{
		\begin{tabular}{|c|lll|lll|lll|lll|}
			\hline
			& \multicolumn{3}{c|}{Gaussian Mixture} & \multicolumn{3}{c|}{Logistic Mixture} & \multicolumn{3}{c|}{Cauchy Mixture} & \multicolumn{3}{c|}{Gamma Mixture} \\
			& Our    & WO/M    & IRA   & Our    & WO/M    & IRA   & Our   & WO/M   & IRA   & Our   & WO/M   & IRA  \\\hline
			\multirow{2}{*}{Wass} &   \textbf{35.0}     &     89.5        &    53.6   &    \textbf{45.8}    &     154        &   68.0    &   \textbf{179}    &    211        &    253   &   \textbf{35.5}    &    510        &   --   \\
			&   $\pm$\textbf{4.68}     &    $\pm$10.0         &  $\pm$21.9     &   $\pm$12.7     &    $\pm$\textbf{10.9}         &  $\pm$24.2     &   $\pm$15.3    &   $\pm$\textbf{7.77}         & $\pm$65.4     &   $\pm$\textbf{1.82}    &     $\pm$71.2       &   --   \\\hline
			\multirow{2}{*}{NLL}  &   \textbf{11.8}     &     12.2        &  11.8    &   \textbf{10.6}     &      12.5       &   10.6    &  12.1     &    12.3        &   \textbf{12.0}    &   \textbf{13.3}    &     19.1       &   --   \\
			&   $\pm$\textbf{0.01}     &    $\pm$0.11        &  $\pm$0.05     &   $\pm$\textbf{0.02}     &    $\pm$0.25        &   $\pm$0.06    &  $\pm$\textbf{0.01}     &      $\pm$0.01      &   $\pm$0.05    &   $\pm$\textbf{0.07}    &     $\pm$0.86      &  --    \\\hline
			{PSNR} &    \textbf{18.9}    &    18.6         &   18.3    &  \textbf{18.8}      &    18.2         &  17.8     &  19.3     &    \textbf{19.3}        &  18.4     &   \textbf{21.0}    &     18.0       &  --    \\
			(dB)&   $\pm$\textbf{0.07}     &     $\pm$0.18        &   $\pm$0.66    &    $\pm$\textbf{0.11}    &     $\pm$0.18        &   $\pm$0.67    &  $\pm$\textbf{0.05}     &    $\pm$0.06        &   $\pm$0.56    &   $\pm$\textbf{0.08}    &     $\pm$0.58       &   --   \\\hline
			\multirow{2}{*}{SSIM} &   \textbf{0.68}     &    0.65         &   0.66    &   \textbf{0.67}     &    0.62         &   0.63    &  \textbf{0.70}     &     0.70       &   0.69    &   \textbf{0.73}    &     0.59       &   --   \\
			&   $\pm$\textbf{0.00}     &      $\pm$0.01       &   $\pm$0.03    &   $\pm$\textbf{0.01}     &     $\pm$0.01        &  $\pm$0.03     &  $\pm$\textbf{0.00}     &    $\pm$0.00        &   $\pm$0.02    &  $\pm$\textbf{0.00}     &     $\pm$0.03       &  --    \\\hline
			FailR      &    \textbf{0.20}$\%$    &     53.1$\%$        &   2.98$\%$    &   \textbf{0.46}$\%$     &     75.5$\%$        &    2.68$\%$   &   \textbf{1.16}$\%$    &  0.82$\%$          &   17.16$\%$    &   \textbf{0}$\%$    &    2.02$\%$        &   100$\%$ \\\hline
		\end{tabular}\label{imagecomp}}
\end{table*}

We evaluated the effectiveness of our manifold and the proposed \textit{Dadam} on both synthetic data and image data, by employing $4$ EMMs, i.e., mixtures of \textit{Gaussian}, \textit{Logistic}, \textit{Cauchy} and \textit{Gamma} ($s=1, a=2, b = 0.5$ in Table \ref{typicalRES}).

\textbf{Synthetic Data:} Each synthetic dataset contains different mixtures of Gaussian distributions, with $10,000$ samples in total. Specifically, we employed three types of synthetic datasets, i.e., $\{m=2, k=3\}$, $\{m=8, k=9\}$ and $\{m=16, k=27\}$. Each type contains $10$ randomly generated mixture datasets (the eccentricity $\varepsilon$ and the separation $c$ of \cite{dasgupta1999learning} were equal to $10$). For every algorithm, we tested on each dataset over $10$ random initialisations and recorded the mean and standard deviation. Then, we averaged over all datasets to obtain the performance of each algorithm. 

\textbf{Image Data:} We adopted the MNIST \cite{lecun2010mnist} dataset as well as the BSDS500 \cite{amfm_pami2011} benchmark dataset in our evaluation. For the MNIST dataset, each image in both the training and testing sets was downsampled to $3\times3$ and $5\times5$; by vectorising each we obtained two test data with $(n\times m)$ being $(70,000\times9)$ and $(70,000\times25)$, respectively. The evaluation on the BSDS500 dataset is slightly different from that on the MNIST dataset, which enables to verify our method on some basic tasks such as the image reconstruction. Specifically, instead of modelling over the whole dataset, each image in the dataset was treated as one test data; this means we had $500$ test data in the BSDS500 dataset. Furthermore, each test data in both the MNIST and BSDS500 datasets was tested for every algorithm with $10$ random initialisations, where the mean and standard deviation were also recorded.

\textbf{Parameter Settings and Metrics:} We found the best learning rates $\alpha$ by searching from  $\{0.001,0.003,0.01,0.03,0.1,0.3\}$. Similar to \cite{kingma2014adam,becigneul2018riemannian}, $\beta_{11} = 0.9$ and $\beta_{2} = 0.999$. The maximum number of iterations for testing the synthesis data was $2,000$ and that for the image data was $10,000$. The Wasserstein distance (simplified as\textit{ Wass}) and averaged negative log-likelihood (\textit{NLL}) were mainly employed for comparison. The optimisation fail ratio (\textit{FailR}) was also reported to show the stability of optimisation under various initialisations. For image data, we adopted two well-known metrics, namely, the peak signal-to-noise ratio (\textit{PSNR}) and structural similarity index (\textit{SSIM}) \cite{wang2004image},  to evaluate the quality of reconstructed images via maximising the posterior of optimised EMMs.

\textbf{Baselines:} Each main part of our method was assessed. The vanilla Riemannian gradient descent method on our manifold is denoted by \textit{W/M + Vanilla}. The Riemannian \textit{Adam} of \cite{becigneul2018riemannian} with our manifold is \textit{W/M + Adam}. Our proposed \textit{Dadam} is denoted by \textit{W/M + Dadam} or \textit{Our} interchangeably. We also performed comparisons with the trivial gradient descent over the sliced Wasserstein distance without the manifold (denoted by \textit{WO/M}), \textit{which is the basics of \cite{kolouri2018sliced}}. The EM-type methods were also compared, which are denoted as \textit{EM} for GMMs and \textit{IRA} for other EMMs \cite{kent1991redescending}. All the methods were run on the Matlab 2017a under Intel Core(TM) i7-6700 CPU, where the time was recorded.

\subsection{Assessment over the \textit{Dadam} on Synthesis Data}
We first evaluated the effectiveness of our $Dadam$ in adaptively accelerating the convergence. In this part, we only optimised the $\mathbf{\Sigma}_i$ whilst setting the $\bm{\pi}$ and $\bm{\mu}_i$ as the ground truth of the synthesis data. Fig. \ref{convergence} illustrates the convergence speed; observe that the optimisation on our established manifold exhibits much faster convergence. Furthermore, both the adaptive methods of \cite{becigneul2018riemannian} and the $Dadam$ were shown to boost the speed. The proposed $Dadam$ achieved the fastest convergence, which verifies the effectiveness of our \textit{Dadam}.

\subsection{Overall Comparisons on Synthesis Data}
We next compared our algorithm (i.e., \textit{W/M+Dadam}) with the \textit{WO/M} and the \textit{EM} methods under GMM problems. Similar results can be found for other EMMs; the result for GMMs is given in Table \ref{GMMcomp}. It should be pointed out that all parameters ($\bm{\pi}$, $\bm{\mu}_i$ and $\mathbf{\Sigma}_i$) were optimised in this test and each result was reported via the corresponding best learning rates. From this table, observe that our algorithm not only achieves an extremely stable estimation (the lowest deviations) but also the lowest values of both the Wasserstein distance and the NLL. An illustrative example can be found in Fig. \ref{random_trials}, in which our algorithm consistently achieves the optimal clustering but the optimised GMMs via the \textit{EM} are highly unstable. From Table \ref{GMMcomp}, we should also point out that although our algorithm implements a set of manifold operations, the computational burden over that without manifold optimisation (i.e., \textit{WO/M}) is in average around $5\%$. However, our algorithm enjoys much faster speed compared to the \textit{EM} algorithm per iteration. 
\begin{figure}[!t]
	\begin{center}
		\includegraphics[width=0.6\columnwidth]{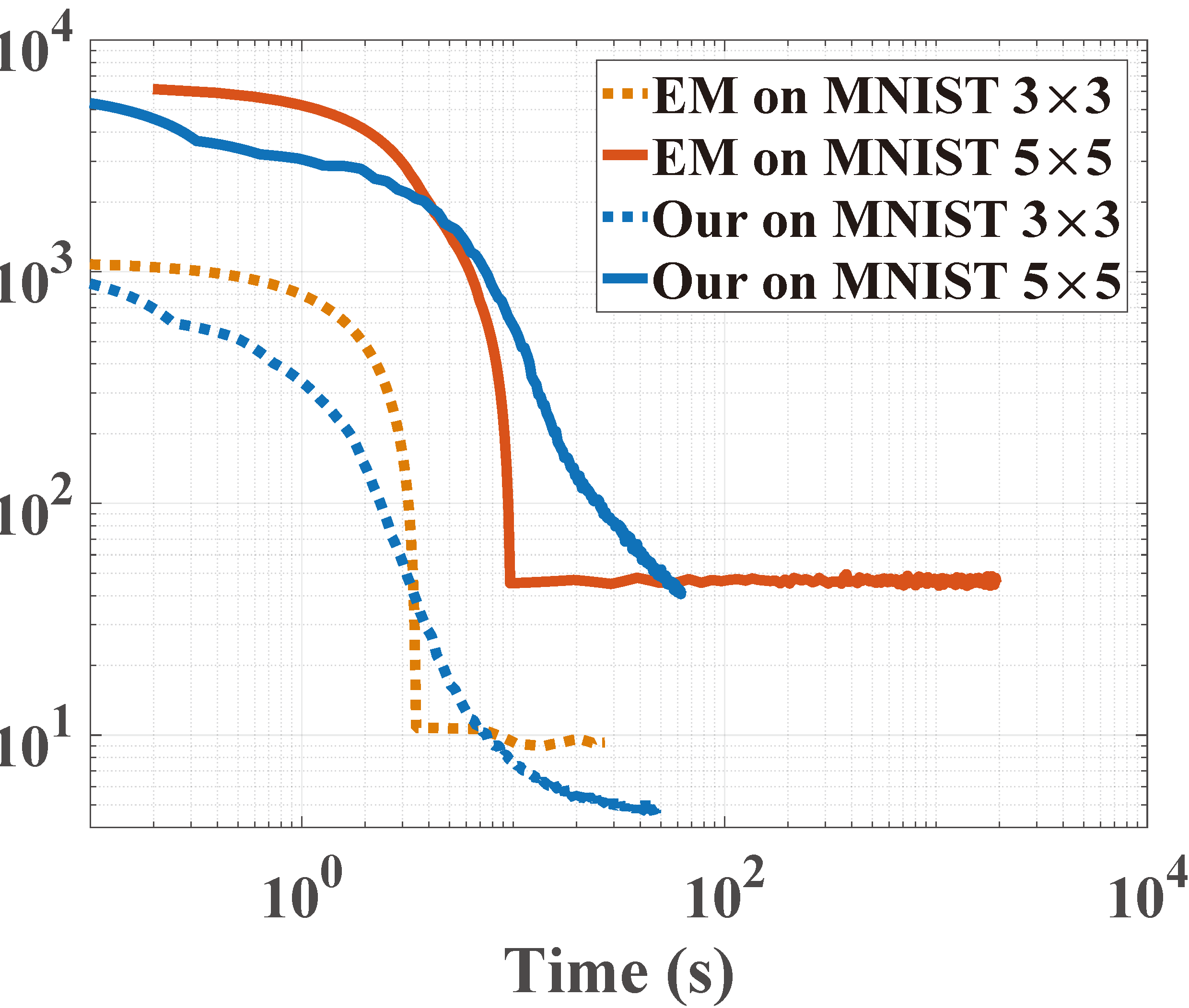}
	\end{center}
	\caption{Wasserstein distance against the computational time for the \textit{EM} and \textit{Our}, on the downsampled $3\times3$ and $5\times5$ MNIST test data. The best learning rate is chosen as $\alpha=0.1$. The cluster number is chosen as $k=3$. The maximal iterations are set to $10,000$ for both methods.}\label{mnist}
\end{figure}

\subsection{Overall Comparisons on the MNIST}
We further compared the computational complexity of our algorithm (i.e., \textit{W/M+Dadam}) with the \textit{EM} method, besides the time per iteration reported in Table \ref{GMMcomp}. Fig. \ref{mnist} plots the decrease of Wasserstein loss with regard to the computational time when optimising GMMs and the trends for other EMMs are the same. We can see clearly from this figure that our algorithm consistently achieves the lowest Wasserstein cost. The computational time of our algorithm is comparable to that of the \textit{EM} algorithm when $m=3\times3=9$; our algorithm, however, requires much less time when in higher dimensions such as $m=5\times5=25$. We shall also need to point out that the random projections in our algorithm can be implemented parallelly, which can further accelerate the convergence speed of our algorithm.

\subsection{Overall Comparisons on the BSDS500}
Finally, we evaluated our algorithm on the BSDS500 dataset, with the results shown in  Table \ref{imagecomp} shows the results. From this table, we see some improvements of using the Wasserstein distance instead of the KL divergence, by comparing the \textit{WO/M} with the \textit{IRA}. However, due to the probability constraint, the \textit{WO/M} is highly unstable and its fail ratio reaches $>50\%$ for the Gaussian mixture and Logistic mixture. More importantly, by optimising on our established manifold together with our \textit{Dadam} method (denoted as \textit{Our}), we can consistently achieve best performances over the five metrics. Again, our algorithm also achieve highly stable estimation, taking advantages from the Wasserstein distance, compared to the \textit{IRA} which operates under the KL. Another advantage of our algorithm is its universal convergence, while the \textit{IRA} does not converge for the mixture of \textit{Gamma}, which limits the flexibility of EMMs.

\section{Conclusion}
We have introduced a new and complete framework for solving general EMMs. To this end, we have first provided a unified and easy-to-use summary of the candidates for identifiable EMMs, which is achieved via a stochastic representation. Then, an approximate Wasserstein distance for EMMs has been proposed, which unlike the existing Wasserstein distances allows the corresponding metrics to be explicitly calculated. The so established manifold has been shown to consistently improve performances in terms of Wasserstein cost and even the NLL cost, and also significantly stabilise the optimisation on EMMs, making them robust to initialisations. We have further proposed a directional adaptively accelerated algorithm to enhance and stabilise the convergence, the performance of which has been validated through comprehensive experimental results.

\section{Appendix}
\subsection{Proof for expressions in Table \ref{typicalRES}}\label{proftable1}
Consider the term $p_{\mathcal{R}}(t)= g(t^2) \cdot t^{m-1}$. The pdf of $\mathcal{R}^2$ is then obtained accordingly as,
\begin{equation}
p_{\mathcal{R}^2}(t)= \frac{1}{2}\cdot g(t) \cdot t^{\nicefrac{m}{2}-1}.
\end{equation}

The term $g(t)$ can be further decomposed as $c\cdot g_c(t)$, where $g_c(t)$ is a nuclear term that only relates to $t$, and $c$ is a normalisation term of $p_{\mathcal{R}}(t)$. This allows us to prove the results in Table 1 related to computable elliptical probability distribution functions (pdf) as follows. 
\begin{itemize}
	\item \textbf{Kotz type distributions:} The nuclear term of density generator is $g_c(t) = t^{a-1} \mathrm{exp}(-bt^s)$ and $c$ can be calculated from  the condition  $\int p_{\mathcal{R}}(t) dt = \int g(t^2) \cdot t^{m-1} dt = 1$ as,
	\begin{equation}
	c = ( \int_{0}^{\infty} t^{2a+m-3} \mathrm{exp}(-bt^{2s}) dt)^{-1} =  \frac{2sb^{\frac{2a+m-2}{2s}}}{\Gamma(\frac{2a+m-2}{2s})}.
	\end{equation}
	We now arrive at the pdf of $\mathcal{R}^2$ for the Kotz type distributions, in the form
	\begin{equation}\label{Kotzpdf}
	p_{\mathcal{R}^2}(t) = \frac{sb^{\frac{2a+m-2}{2s}}}{\Gamma(\frac{2a+m-2}{2s})}t^{\frac{m}{2}+a-2}\mathrm{exp}(-bt^s).
	\end{equation}
	However, in practice, it is intractable to generate samples of $\mathcal{R}^2$ as in \eqref{Kotzpdf}. On the other hand, it can be found that $\mathcal{G}^{1/s}$ has the same distribution as $\mathcal{R}^2$ when $\mathcal{G}$ is Gamma distributed. Specifically, when $\mathcal{G}\sim \mathrm{Ga}(\frac{2a+m-2}{2s}, b)$, we obtain the pdf of $\mathcal{G}^{1/s}$ as,
	\begin{equation}
	\begin{aligned}
	p_{\mathcal{G}^{1/s}}(t) &= s t^{s-1} \frac{b^{\frac{2a+m-2}{2s}}t^{\frac{2a+m-2}{2}-s}\mathrm{exp}(-bt^s)}{\Gamma(\frac{2a+m-2}{2s})}\\
	& = \frac{sb^{\frac{2a+m-2}{2s}}}{\Gamma(\frac{2a+m-2}{2s})}t^{\frac{m}{2}+a-2}\mathrm{exp}(-bt^s).
	\end{aligned}
	\end{equation}
	This proves that $\mathcal{R}^2=^d\mathcal{G}^{1/s}$ for the Kotz type distributions.
	
	\item \textbf{Scale mixture of normals distributions:} The scale mixture of normal distributions consists of a mixture of zero-mean normal distributions (denoted as $\bm{\mathcal{X}} = \sqrt{\mathcal{K}}\bm{\mathcal{N}}$, where $\bm{\mathcal{N}}$ is zero-mean normal distribution and $\mathcal{K}$ is called the mixing distribution with pdf $p_\mathcal{K}(t)$). Correspondingly, we can write the pdf of $\bm{\mathcal{X}}$ as
	\begin{equation}\label{scalenorm}
	p_{\bm{\mathcal{X}}}(\mathbf{x})\! =\! \int_{t}\!p_{\bm{\mathcal{X}}}(\mathbf{x}|t)p_{\mathcal{K}}(t)dt \!\propto\!\! \int_{t} \!t^{-\frac{m}{2}}\mathrm{exp}(\frac{\mathbf{x}^T \mathbf{\Sigma}^{-1} \mathbf{x}}{t})p_{\mathcal{K}}(t)dt.
	\end{equation}
	Furthermore, to represent a scale mixture of normal distributions in the form of stochastic representations, by inspection we see that a normal distribution, $\bm{\mathcal{N}}$, can be represented by a multiplication of $\sqrt{\chi_m^2}$ and $\bm{\mathcal{S}}$ (from the stochastic representation), so that the following holds:
	\begin{equation}\label{scalenormsto}
	\bm{\mathcal{X}} = \sqrt{\mathcal{K}}\bm{\mathcal{N}} = \sqrt{\mathcal{K} \cdot \chi_m^2}\cdot \bm{\mathcal{S}} = \sqrt{\mathcal{K} \cdot \mathcal{G}}\bm{\mathcal{S}},
	\end{equation}
	where $\mathcal{G}\sim \mathrm{Ga}(m/2,2)$. In this case, $\mathcal{R}^2 =^d \mathcal{K} \cdot \mathcal{G}$ by the stochastic representation. Therefore, although \eqref{scalenorm} and \eqref{scalenormsto} are equivalent, \eqref{scalenormsto} provides a unified stochastic representation within the elliptical family. This allows us to discuss different $\mathcal{K}$ for different sub-classes of scale mixture normals in the following.
	
	For the Pearson type VII distributions, where the nuclear density generator is given by $g_c(t) = (1+\nicefrac{t}{v})^{-s}$, we can obtain the corresponding normalisation term $c$ as follows,
	\begin{equation}
	\begin{aligned}
	c &= (\int_{0}^{\infty} t^{m-1} (1+t^2/v)^{-s} dt)^{-1}\\
	 &= \frac{2\Gamma(s)}{( v)^{m/2}\Gamma(s-m/2)\Gamma(m/2)}.
	\end{aligned}
	\end{equation}
	
	Consequently, the pdf of $\mathcal{R}^2$ becomes:
	\begin{equation}\label{VIIIr}
	p_{\mathcal{R}^2}(t) = \frac{\Gamma(s)}{\Gamma(s-\frac{m}{2})\Gamma(\frac{m}{2})v^{\frac{m}{2}}}(1+\frac{t}{v})^{-s}t^{\frac{m}{2}-1}.
	\end{equation}
	
	Again, the rather complicated form of \eqref{VIIIr} makes it impossible to generate samples. It can be found that when $p_\mathcal{K}(t) = \frac{(\frac{v}{2})^{s-\nicefrac{m}{2}}}{\Gamma(s-\nicefrac{m}{2})}t^{\nicefrac{m}{2}-s-1}\mathrm{exp}(-\nicefrac{v}{2t})$ and $\mathcal{G}\sim \mathrm{Ga}(m/2,2)$, $\mathcal{R}^2$ in \eqref{VIIIr} has the same pdf as $\mathcal{K} \cdot \mathcal{G}$. This can be verified as follows,
	\begin{equation}
	\begin{aligned}
	&p_{\mathcal{R}^2}(t) = \int_{0}^{\infty}p_\mathcal{G}(t|\tau)p_\mathcal{K}(\tau)d\tau \\
	&\!=\! \int_{0}^{\infty}\!\! \frac{(\frac{2}{\tau})^{\frac{m}{2}}t^{\frac{m}{2}-1}\mathrm{exp}(-\frac{t}{2\tau})}{\Gamma(\frac{m}{2})}\frac{(\frac{v}{2})^{s-\frac{m}{2}}}{\Gamma(s-\frac{m}{2})}\tau^{\frac{m}{2}-s-1}\mathrm{exp}(-\frac{v}{2\tau})d\tau\\
	&=\frac{\Gamma(s)t^{\frac{m}{2}-1}}{\Gamma(s-\frac{m}{2})\Gamma(\frac{m}{2})v^{\frac{m}{2}}(1+\frac{t}{v})^{s}},
	\end{aligned}
	\end{equation}
	where $\mathcal{K}^{-1}\sim \mathrm{Ga}(s - \frac{m}{2}, \frac{v}{2})$ and $\mathcal{R}^2 =^d \mathcal{K} \cdot \mathcal{G}$.
	
	Moreover, for other types within the scale mixture normals, Barndorff \textit{et al.} \cite{barndorff1982normal} have proved that the generalised hyperbolic distributions can be formulated in the form of \eqref{scalenorm} when $\mathcal{K}$ satisfies the inverse-Gaussian distribution. The elliptical logistic distributions are also mixtures of normals, where $\sqrt{\mathcal{K}}$ in \eqref{scalenorm} relates to the Kolmogorov-Smirnov distribution \cite{stefanski1991normal}. Besides, the relationship between the scale mixture normals and the $\alpha$-stable distribution is given in \cite{andrews1974scale}, which satisfies our equivalent form of \eqref{scalenormsto}. We omit the tedious proofs here.
	
	\item \textbf{Pearson type II distributions:} The nuclear density generator is $g_c(t) = (1-t)^{s-1}$ with the constraint $t\in[0,1]$. We omit the proof here because the calculation of $c$ and the pdf of $\mathcal{R}^2$ are well documented in \cite{fang2018symmetric}.
\end{itemize}
This completes the proof of Table \ref{typicalRES}.

\subsection{Proof of Theorem \ref{Theorem1}}\label{proftheorem1}
\textbf{Symmetry:} It is easy to verify the symmetry of $d_U(\bm{\mathcal{Y}}_1,\bm{\mathcal{Y}}_2)$, i.e., $d_U(\bm{\mathcal{Y}}_1,\bm{\mathcal{Y}}_2)=d_U(\bm{\mathcal{Y}}_2,\bm{\mathcal{Y}}_1)$.

\noindent\textbf{Non-negativity:} Due to $\arccos(\cdot)\geq0$, $d^2_W(\bm{\mathcal{X}}_{i,1},\bm{\mathcal{X}}_{j,2})\geq0$ and $\gamma(i,j)\geq0$ for all $i,j$, $d_U(\bm{\mathcal{Y}}_1,\bm{\mathcal{Y}}_2)\geq0$. More importantly, the equality holds if and only if $\bm{\mathcal{Y}}_1=^d\bm{\mathcal{Y}}_2$. In this case,  $\min_{\gamma(i,j)}\sum_{i,j}\frac{\gamma(i,j)}{k}d^2_W(\bm{\mathcal{X}}_{i,1},\bm{\mathcal{X}}_{j,2}) = 0$ and $\min_{\gamma(i,j)}\arccos(\sum_{i,j}\gamma(i,j)\sqrt{\pi_{i,1}\pi_{j,2}}) = \arccos(\sum_{i}{\pi_{i,1}}) = \arccos(1)=0$, which results in $d_U(\bm{\mathcal{Y}}_1,\bm{\mathcal{Y}}_2) = 0$. 

\noindent\textbf{Triangle inequality:} Because $\gamma(i,j)$ is binary $\in \{0,1\}$ for each pair $\{i,j\}$, and it satisfies $\sum_{i=1}^{k}\gamma(i,j) = 1$ and $\sum_{j=1}^{k}\gamma(i,j) = 1$, the $\gamma(i,j)$ operates as a bijection between the elliptical distributions within the first and second EMMs.

Then, in order to prove the triangle property, we denote the third EMM as $\bm{\mathcal{Y}}_3=^d\sum_{h=1}^{k}z_{h,3}\bm{\mathcal{X}}_{h,3}$, and investigate the relationship between $d_U(\bm{\mathcal{Y}}_1,\bm{\mathcal{Y}}_2)$ and $d_U(\bm{\mathcal{Y}}_1,\bm{\mathcal{Y}}_3)+d_U(\bm{\mathcal{Y}}_2,\bm{\mathcal{Y}}_3)$. We use $\gamma^*(\cdot,\cdot)$ to denote the optimal $\gamma(\cdot,\cdot)$ in the defined distance $d_U(\cdot,\cdot)$, and also define the following function
\begin{equation}
\begin{aligned}
&f\big(\bm{\mathcal{Y}}_1,\bm{\mathcal{Y}}_2, \gamma(i,j)\big) = \\
& \sum_{i,j}\gamma(i,j)d^2_W(\bm{\mathcal{X}}_{i,1},\bm{\mathcal{X}}_{j,2})+\arccos(\sum_{i,j}\gamma(i,j)\sqrt{\pi_{i,1}\pi_{j,2}}).
\end{aligned}
\end{equation}
By Definition \ref{def}, $\min_{\gamma(i,j)} f(\bm{\mathcal{Y}}_1,\bm{\mathcal{Y}}_2, \gamma(i,j)) = f(\bm{\mathcal{Y}}_1,\bm{\mathcal{Y}}_2, \gamma^*(i,j)) = d_U(\bm{\mathcal{Y}}_1,\bm{\mathcal{Y}}_2)$. 

More importantly, for two arbitrary $\gamma(i,h)$ and $\gamma(h,j)$, their combination $\gamma(i,h)\cap\gamma(h,j) = \sum_{h=1}^k\gamma(i,h)\cdot\gamma(h,j)$ also formulates a transport plan $\gamma(i,j)$ because  $\gamma(i,h)\cap\gamma(h,j)$ is still binary and $\sum_{i=1}^{k}\gamma(i,h)\cap\gamma(h,j) = \sum_{j=1}^{k}\gamma(i,h)\cap\gamma(h,j) = 1$.

We therefore now arrive at
\begin{equation}\label{proofdistnace}
\begin{aligned}
&d_U(\bm{\mathcal{Y}}_1,\bm{\mathcal{Y}}_3) + d_U(\bm{\mathcal{Y}}_2,\bm{\mathcal{Y}}_3) \\
&= f(\bm{\mathcal{Y}}_1,\bm{\mathcal{Y}}_3, \gamma^*(i,h)) + f(\bm{\mathcal{Y}}_2,\bm{\mathcal{Y}}_3, \gamma^*(j,h))\\ 
&=\frac{1}{k}\big(\!\sum_{i,h}\!\gamma^*(i,h)d^2_W(\bm{\mathcal{X}}_{i,1},\!\bm{\mathcal{X}}_{h,3})\! +\!\! \sum_{h,j}\!\gamma^*(h,j)d^2_W(\bm{\mathcal{X}}_{h,3},\!\bm{\mathcal{X}}_{j,2})\big) \\
&~\!+ \!\arccos(\!\sum_{i,h}\!\gamma^*(i,h)\sqrt{\pi_{i,1}\pi_{h,3}})\! +\! \arccos(\!\sum_{h,j}\!\gamma^*(h,j)\sqrt{\pi_{h,3}\pi_{j,2}})\\
&\geq \frac{1}{k}\sum_{i,j}\big(\gamma^*(i,h)\cap\gamma^*(h,j)\big)d^2_W(\bm{\mathcal{X}}_{i,1},\bm{\mathcal{X}}_{j,2}) \\
&~~+ \arccos\big(\sum_{i,j}(\gamma^*(i,h)\cap\gamma^*(h,j))\sqrt{\pi_{i,1}\pi_{j,2}}\big)\\
&=f\big(\bm{\mathcal{Y}}_1,\bm{\mathcal{Y}}_2, \gamma^*(i,h)\cap\gamma^*(h,j)\big)\\
&\geq f\big(\bm{\mathcal{Y}}_1,\bm{\mathcal{Y}}_2, \gamma^*(i,j)\big) \\
&= d_U(\bm{\mathcal{Y}}_1,\bm{\mathcal{Y}}_2).
\end{aligned}
\end{equation}
In \eqref{proofdistnace}, the first inequality holds due to the fact that both $d^2_W(\bm{\mathcal{X}}_{i,1},\bm{\mathcal{X}}_{j,2})$ and $\arccos(\sum_{i,j}\gamma(i,j)\sqrt{\pi_{i,1}\pi_{j,2}})$ satisfy the triangle property. Moreover, the second inequality is due to the fact that the combined plan $\gamma(i,h)\cap\gamma(h,j)$  is not necessarily the optimal plan between $\bm{\mathcal{Y}}_1$ and $\bm{\mathcal{Y}}_2$, as the optimal plan $\gamma^*(i,j)$ achieves the minimum and defines the distance of $d_U(\bm{\mathcal{Y}}_1,\bm{\mathcal{Y}}_2)$.

This completes the proof of Theorem \ref{Theorem1}.

\subsection{Proof of Lemma \ref{Lemma1}}\label{proflemma1}
Recall that $\bm{\mathcal{Y}}\sim\sum_{i=1}^{k} \pi_i \mathcal{E}(\mathbf{x}|\bm{\mu}_i, \mathbf{\Sigma}_i, \mathcal{R})$. We can now write the definition of the Wasserstein distance in the form
\begin{equation}\label{wasscom1}
d^2_W(\bm{\mathcal{Y}}_1,\bm{\mathcal{Y}}_2) \!=\!\! \inf_{\eta(\bm{\mathcal{Y}}_1,\bm{\mathcal{Y}}_2)}\!\int_{m\times m}\! \eta(\bm{\mathcal{Y}}_1,\bm{\mathcal{Y}}_2)||\mathbf{x}_1 - \mathbf{x}_2||_2^2d\mathbf{x}_1d\mathbf{x}_2,
\end{equation}
where $\eta(\bm{\mathcal{Y}}_1,\bm{\mathcal{Y}}_2)$ denotes the joint distribution between $\bm{\mathcal{Y}}_1$ and $\bm{\mathcal{Y}}_2$, and satisfies $\int_m\eta(\bm{\mathcal{Y}}_1,\bm{\mathcal{Y}}_2) d\mathbf{x}_1 = p_{\bm{\mathcal{Y}}_2}(\mathbf{x}_2)$ and $\int_m\eta(\bm{\mathcal{Y}}_1,\bm{\mathcal{Y}}_2) d\mathbf{x}_2 = p_{\bm{\mathcal{Y}}_1}(\mathbf{x}_1)$.

More importantly, 
\begin{equation}\label{wasscom2}
\begin{aligned}
&\sum_{i,j}\frac{\gamma^*(i,j)}{k}d^2_W(\bm{\mathcal{X}}_{i,1},\bm{\mathcal{X}}_{j,2}) =\\
& \!\sum_{i,j}\!\frac{\gamma^*(i,j)}{k}\!\!\inf_{\eta(\bm{\mathcal{X}}_{i,1},\bm{\mathcal{X}}_{j,2})}\!\int_{m\times m}\! \eta(\bm{\mathcal{X}}_{i,1},\bm{\mathcal{X}}_{j,2})||\mathbf{x}_1\! -\! \mathbf{x}_2||_2^2d\mathbf{x}_1d\mathbf{x}_2\\
& = \int_{m\times m} \sum_{i,j}\frac{\gamma^*(i,j)}{k}\eta^*(\bm{\mathcal{X}}_{i,1},\bm{\mathcal{X}}_{j,2})||\mathbf{x}_1 - \mathbf{x}_2||_2^2d\mathbf{x}_1d\mathbf{x}_2,
\end{aligned}
\end{equation}
where $\eta^*(\bm{\mathcal{X}}_{i,1},\bm{\mathcal{X}}_{j,2})$ is the optimal plan that achieves the Wasserstein distance between $\bm{\mathcal{X}}_{i,1}$ and $\bm{\mathcal{X}}_{j,2}$. Therefore, by comparing \eqref{wasscom1} and \eqref{wasscom2}, we can easily observe that $\sum_{i,j}\frac{\gamma^*(i,j)}{k}\eta^*(\bm{\mathcal{X}}_{i,1},\bm{\mathcal{X}}_{j,2})$ consists a subset of joint distribution of $\eta(\bm{\mathcal{Y}}_1,\bm{\mathcal{Y}}_2)$, due to $\pi_{i,1} = \pi_{j,2} = \nicefrac{1}{k}$. In other words, because of the factorisation from $\bm{\mathcal{Y}}$ to $\bm{\mathcal{X}}$, the $\sum_{i,j}\frac{\gamma^*(i,j)}{k}\eta^*(\bm{\mathcal{X}}_{i,1},\bm{\mathcal{X}}_{j,2})$ does not necessarily achieve the optimal transport plan between $\bm{\mathcal{Y}}_1$ and $\bm{\mathcal{Y}}_2$. 

Moreover, for balanced EMMs, the following holds for arbitrary $\gamma(i,j)$,
\begin{equation}
\arccos(\sum_{i,j}\gamma(i,j)\sqrt{\pi_{i,1}\pi_{j,2}}) = 0.
\end{equation}

Thus, we have
\begin{equation}
d^2_W(\bm{\mathcal{Y}}_1,\bm{\mathcal{Y}}_2) \leq \sum_{i,j}\frac{\gamma^*(i,j)}{k}d^2_W(\bm{\mathcal{X}}_{i,1},\bm{\mathcal{X}}_{j,2}) = d_U(\bm{\mathcal{Y}}_1,\bm{\mathcal{Y}}_2).
\end{equation}

This completes the proof of Lemma \ref{Lemma1}.

\subsection{Proof of Lemma \ref{manifold}}\label{proflemma2}
Due to the fact that $\gamma(i,j)$ in Definition 1 is a bijection between the mixture components in $\bm{\mathcal{Y}}_1$ and $\bm{\mathcal{Y}}_2$, for each $\bm{\pi}_{1}$, $\bm{\mu}_{i,1}$ and $\mathbf{\Sigma}_{i,1}$ in $\bm{\mathcal{Y}}_1$, there exist only one corresponding $\bm{\pi}_{2}$, $\bm{\mu}_{j,2}$ and $\mathbf{\Sigma}_{j,2}$, respectively, in $\bm{\mathcal{Y}}_2$. In other words, when casting the problem onto the parameter space, the Hessian of $d_U(\bm{\mathcal{Y}}_{\bm{\theta}}, \bm{\mathcal{Y}}_{\bm{\theta}+s\Delta\bm{\theta}})$ can be calculated as follows,
\begin{equation}\label{prodm}
\begin{aligned}
&\frac{\partial^2 d_U(\bm{\mathcal{Y}}_{\bm{\theta}}, \bm{\mathcal{Y}}_{\bm{\theta}+s\Delta\bm{\theta}})}{\partial s^2}|_{s \rightarrow 0} \\
&= \frac{1}{k} \sum_{i=1}^K\frac{\partial^2 d^2_W(\bm{\mathcal{X}}_{\bm{\mu}_i, \mathbf{\Sigma}_i}, \bm{\mathcal{X}}_{\bm{\mu}_i + s\Delta \bm{\mu}_i, \mathbf{\Sigma}_i+s\Delta\mathbf{\Sigma}_i})}{\partial s^2}|_{s \rightarrow 0}\\
&~~~~ + \frac{\partial^2 \arccos(\sqrt{\bm{\pi}}^T(\sqrt{\bm{\pi}+s\Delta\bm{\pi}}))}{\partial s^2}|_{s\rightarrow 0},
\end{aligned}
\end{equation}
where the square root operation $\sqrt{\cdot}$ is performed element-wise.

In \eqref{prodm}, we can see that the manifold defined by $d_U(\bm{\mathcal{Y}}_{\bm{\theta}}, \bm{\mathcal{Y}}_{\bm{\theta}+\Delta\bm{\theta}})$ is a product manifold defined by $d^2_W(\bm{\mathcal{X}}_{\bm{\mu}_i, \mathbf{\Sigma}_i}, \bm{\mathcal{X}}_{\bm{\mu}_i + \Delta \bm{\mu}_i, \mathbf{\Sigma}_i+\Delta\mathbf{\Sigma}_i})$ and $\arccos(\sqrt{\bm{\pi}}^T(\sqrt{\bm{\pi}+\Delta\bm{\pi}}))$.

Furthermore, $\arccos(\sqrt{\bm{\pi}}^T(\sqrt{\bm{\pi}+\Delta\bm{\pi}}))$ defines a sphere manifold \cite{absil2009optimization} (Examples 3.5.1 and 3.6.1), for which basic operations are provided in Examples 4.1.1, 5.4.1 and 8.1.7 of \cite{absil2009optimization}. On the other hand, the Wasserstein distance between two elliptical distributions has an explicit representation as follows \cite{ghaffari2018multivariate,muzellec2018generalizing},
\begin{equation}\label{ellipdist}
\begin{aligned}
&d^2_W(\bm{\mathcal{X}}_{\bm{\mu}_i, \mathbf{\Sigma}_i}, \bm{\mathcal{X}}_{\bm{\mu}_i + \Delta \bm{\mu}_i, \mathbf{\Sigma}_i+\Delta\mathbf{\Sigma}_i}) =||\Delta\bm{\mu}_i||_2^2 \\
&\!+\! \frac{\mathbb{E}[\mathcal{R}^2]}{m}\mathrm{tr}(\mathbf{\Sigma}_{i} \!+\! (\mathbf{\Sigma}_i\!+\!\Delta\mathbf{\Sigma}_i) \!-\! 2(\mathbf{\Sigma}_{i}^{\nicefrac{1}{2}}(\mathbf{\Sigma}_i\!+\!\Delta\mathbf{\Sigma}_i)\mathbf{\Sigma}_{i}^{\nicefrac{1}{2}})^{\nicefrac{1}{2}}).
\end{aligned}
\end{equation}
We can thus conclude that the manifold for $\bm{\mu}_i$ is the conventional Euclidean space within $\mathbb{R}^m$. Moreover, the manifold of $\mathbf{\Sigma}_i$ can also be correspondingly obtained on the basis of \cite{malago2018wasserstein}.

This completes the proof of Lemma \ref{manifold}.

\subsection{Calculations within Table \ref{basciops}}\label{proftable2}
The basic operations on a sphere manifold were described in \cite{absil2009optimization} (Examples 3.5.1, 3.6.1, 4.1.1, 5.4.1 and 8.1.7). As the manifold for $\bm{\mu}_i$ is the conventional Euclidean space, the trivial gradient descent can be employed. Operations for $\mathbf{\Sigma}_i$ under the Wasserstein manifold can be found in \cite{malago2018wasserstein}. Furthermore, we omit the scale weight $\frac{\mathbb{E}[\mathcal{R}^2]}{m}$ in our work as it can be incorporated into the stepsize $\alpha$ during the gradient descent.

To calculate the Euclidean gradients, we first explicitly express our cost function, i.e., the sliced Wasserstein distance $d_{SW}(\bm{\mathcal{Y}}_{\bm{\theta}},\bm{\mathcal{Y}}^*)$, as follows,
\begin{equation}\label{costSW}
\begin{aligned}
&d_{SW}(\bm{\mathcal{Y}}_{\bm{\theta}},\bm{\mathcal{Y}}^*) = \\
&\!\int_m \!\inf_{\eta(\mathrm{Ra}(\bm{\mathcal{Y}}_{\bm{\theta}}, \mathbf{p}), \mathrm{Ra}(\bm{\mathcal{Y}}^*, \mathbf{p}))} \!\int_{\mathbb{R}\times\mathbb{R}} \!\eta_R\cdot (y_1\! -\! y_2)^2 \!dy_1dy_2 d\mathbf{p},\\
\end{aligned}
\end{equation}
and
\begin{equation}
\eta_R = \eta(\mathrm{Ra}(\bm{\mathcal{Y}}_{\bm{\theta}}, \mathbf{p}), \mathrm{Ra}(\bm{\mathcal{Y}}^*, \mathbf{p})),
\end{equation}
where $y_1$ denotes the Radon transform \cite{rabin2011wasserstein} $\mathrm{Ra}(\bm{\mathcal{Y}}_{\bm{\theta}}, \mathbf{p})$, which projects $\bm{\mathcal{Y}}_{\bm{\theta}}$ onto a one-dimensional random variable along the direction $\mathbf{p}$. Similarly, $y_2$ denotes the Radon transform of $\bm{\mathcal{Y}}^*$ with the direction $\mathbf{p}$. Recall that $\eta(\cdot,\cdot)$ is a joint distribution. Due to the fact that elliptical distributions belong to the location-scale family, as described by the stochastic representation, the Radon transform of each  $\bm{\mathcal{X}}_{\bm{\mu}_i,\mathbf{\Sigma}_i}$ has a simple representation in the form  ${\mathcal{X}}_{\mathbf{p}^T\bm{\mu}_i,\mathbf{p}^T\mathbf{\Sigma}_i\mathbf{p}} \sim \mathcal{E}(x; \mathbf{p}^T\bm{\mu}_i, \mathbf{p}^T\mathbf{\Sigma}_i\mathbf{p}, \mathcal{R})$, where $\mathcal{X}$ is a one-dimensional elliptical distribution. 

Fortunately, the one-dimensional Wasserstein distance has the closed-form solution \cite{kolouri2017optimal}, so that we can re-write \eqref{costSW} as
\begin{equation}\label{SWdis}
\begin{aligned}
&\min_{\bm{\pi},\bm{\mu}_i,\mathbf{\Sigma}_i}d_{SW}(\bm{\mathcal{Y}}_{\bm{\theta}},\bm{\mathcal{Y}}^*)= \\
&\min_{\bm{\pi},\bm{\mu}_i,\mathbf{\Sigma}_i}\!\int_m \!\int_{\mathbb{R}}\! |y\!-\!T(y)|^2 \sum_{i=1}^k \frac{c_m\pi_i}{(\mathbf{p}^T\mathbf{\Sigma}^h_i\mathbf{p})^{\frac{1}{2}}} g\!\left(\!\frac{(y \!-\! \mathbf{p}^T\bm{\mu}_i)^2}{\mathbf{p}^T\mathbf{\Sigma}^h_i\mathbf{p}}\!\right) \!dy d\mathbf{p},
\end{aligned}
\end{equation}
where $T(y)$ is the optimal transport plan between $\mathrm{Ra}(\bm{\mathcal{Y}}_{\bm{\theta}}, \mathbf{p})$ and $\mathrm{Ra}(\bm{\mathcal{Y}}^*, \mathbf{p})$, which can be explicitly obtained via their cumulative distribution functions \cite{chen2018natural}. When optimising \eqref{SWdis} using stochastic gradient descent, for each $\mathbf{p}$, we therefore minimise 
\begin{equation}\label{spSWdis}
\begin{aligned}
&\min_{\bm{\pi},\bm{\mu}_i,\mathbf{\Sigma}_i}J(\bm{\theta}, \mathbf{p}) = \\
&\min_{\bm{\pi},\bm{\mu}_i,\mathbf{\Sigma}_i} \int_{\mathbb{R}} |y-T(y)|^2 \sum_{i=1}^k \frac{c_m\pi_i}{(\mathbf{p}^T\mathbf{\Sigma}^h_i\mathbf{p})^{\frac{1}{2}}}g\!\left(\!\frac{(y\! -\! \mathbf{p}^T\bm{\mu}_i)^2}{\mathbf{p}^T\mathbf{\Sigma}^h_i\mathbf{p}}\!\right) \!dy.
\end{aligned}
\end{equation}
Then, by calculating the derivatives with regard to $\sqrt{\pi_i}$, $\bm{\mu}_i$ and $\mathbf{\Sigma}_i$, and upon introducing the Kantorovich potential $\phi(y)$ \cite{chen2018natural}, we obtain the Euclidean gradients as follows,
\begin{equation}
\begin{aligned}
&\nabla_{\sqrt{\pi_i},\bm{\mu}_i,\mathbf{\Sigma}_i} J(\bm{\theta}, \mathbf{p})= \\
&\!\int_\mathbb{R}\! \phi(y) \nabla_{\!\sqrt{\pi_i},\bm{\mu}_i,\mathbf{\Sigma}_i} \! \left(\!\sum_{i=1}^k \frac{c_m \sqrt{\pi_i}^2}{(\mathbf{p}^T\mathbf{\Sigma}^h_i\mathbf{p})^{\frac{1}{2}}}g\big(\frac{(y - \mathbf{p}^T\bm{\mu}_i)^2}{\mathbf{p}^T\mathbf{\Sigma}^h_i\mathbf{p}}\big)\right)dy,
\end{aligned}
\end{equation}
where $\nabla_{\sqrt{\pi_i},\bm{\mu}_i,\mathbf{\Sigma}_i}(\cdot)$ denotes the derivatives with regard to $\{\sqrt{\pi_i},\bm{\mu}_i,\mathbf{\Sigma}_i\}$, respectively. The Euclidean gradients in the Table \ref{basciops} can then be easily calculated.

\small
\bibliography{ShengxiLi}
\bibliographystyle{aaai}
\end{document}